\documentclass[10pt,twocolumn,letterpaper]{article}

\usepackage{iccv}
\usepackage{times}
\usepackage{epsfig}
\usepackage{graphicx}
\usepackage{amsmath}
\usepackage{amssymb}
\usepackage{multirow}

\usepackage{bbm}
\usepackage{sidecap}

\usepackage{subfig}

\usepackage{xr}
\usepackage[toc,page]{appendix}

\usepackage{amsthm}

\usepackage{booktabs} 
\usepackage{xcolor}
\usepackage{algorithm,algorithmic}
\usepackage[algo2e,ruled,vlined]{algorithm2e}

\usepackage[pagebackref=true,breaklinks=true,letterpaper=true,colorlinks,bookmarks=false]{hyperref}

\iccvfinalcopy 

\def\iccvPaperID{2870} 

\ificcvfinal\fi

\begin{document}
\def\Blue{\color{blue}}
\def\Purple{\color{purple}}

\def\A{{\bf A}}
\def\a{{\bf a}}
\def\B{{\bf B}}
\def\b{{\bf b}}
\def\C{{\bf C}}
\def\c{{\bf c}}
\def\D{{\bf D}}
\def\d{{\bf d}}
\def\E{{\bf E}}
\def\e{{\bf e}}
\def\f{{\bf f}}
\def\F{{\bf F}}
\def\K{{\bf K}}
\def\k{{\bf k}}
\def\L{{\bf L}}
\def\H{{\bf H}}
\def\h{{\bf h}}
\def\G{{\bf G}}
\def\g{{\bf g}}
\def\I{{\bf I}}
\def\R{{\bf R}}
\def\X{{\bf X}}
\def\Y{{\bf Y}}
\def\OO{{\bf O}}
\def\oo{{\bf o}}
\def\P{{\bf P}}
\def\Q{{\bf Q}}
\def\r{{\bf r}}
\def\s{{\bf s}}
\def\S{{\bf S}}
\def\t{{\bf t}}
\def\T{{\bf T}}
\def\x{{\bf x}}
\def\y{{\bf y}}
\def\z{{\bf z}}
\def\Z{{\bf Z}}
\def\M{{\bf M}}
\def\m{{\bf m}}
\def\n{{\bf n}}
\def\U{{\bf U}}
\def\u{{\bf u}}
\def\V{{\bf V}}
\def\v{{\bf v}}
\def\W{{\bf W}}
\def\w{{\bf w}}
\def\0{{\bf 0}}
\def\1{{\bf 1}}
\def\N{{\bf N}}

\def\AM{{\mathcal A}}
\def\EM{{\mathcal E}}
\def\FM{{\mathcal F}}
\def\TM{{\mathcal T}}
\def\UM{{\mathcal U}}
\def\XM{{\mathcal X}}
\def\YM{{\mathcal Y}}
\def\NM{{\mathcal N}}
\def\OM{{\mathcal O}}
\def\IM{{\mathcal I}}
\def\GM{{\mathcal G}}
\def\PM{{\mathcal P}}
\def\LM{{\mathcal L}}
\def\MM{{\mathcal M}}
\def\DM{{\mathcal D}}
\def\SM{{\mathcal S}}
\def\RB{{\mathbb R}}
\def\EB{{\mathbb E}}

\def\tx{\tilde{\bf x}}
\def\ty{\tilde{\bf y}}
\def\tz{\tilde{\bf z}}
\def\hd{\hat{d}}
\def\HD{\hat{\bf D}}
\def\hx{\hat{\bf x}}
\def\hR{\hat{R}}

\def\Ome{\mbox{\boldmath$\omega$\unboldmath}}
\def\bet{\mbox{\boldmath$\beta$\unboldmath}}
\def\et{\mbox{\boldmath$\eta$\unboldmath}}
\def\ep{\mbox{\boldmath$\epsilon$\unboldmath}}
\def\ph{\mbox{\boldmath$\phi$\unboldmath}}
\def\Pii{\mbox{\boldmath$\Pi$\unboldmath}}
\def\pii{\mbox{\boldmath$\pi$\unboldmath}}
\def\Ph{\mbox{\boldmath$\Phi$\unboldmath}}
\def\Ps{\mbox{\boldmath$\Psi$\unboldmath}}
\def\pss{\mbox{\boldmath$\psi$\unboldmath}}
\def\tha{\mbox{\boldmath$\theta$\unboldmath}}
\def\Tha{\mbox{\boldmath$\Theta$\unboldmath}}
\def\muu{\mbox{\boldmath$\mu$\unboldmath}}
\def\Si{\mbox{\boldmath$\Sigma$\unboldmath}}
\def\Gam{\mbox{\boldmath$\Gamma$\unboldmath}}
\def\gamm{\mbox{\boldmath$\gamma$\unboldmath}}
\def\Lam{\mbox{\boldmath$\Lambda$\unboldmath}}
\def\De{\mbox{\boldmath$\Delta$\unboldmath}}
\def\vps{\mbox{\boldmath$\varepsilon$\unboldmath}}
\def\Up{\mbox{\boldmath$\Upsilon$\unboldmath}}
\def\Lap{\mbox{\boldmath$\LM$\unboldmath}}

\def\tr{\mathrm{tr}}
\def\etr{\mathrm{etr}}
\def\etal{{\em et al.\/}\,}
\newcommand{\indep}{{\;\bot\!\!\!\!\!\!\bot\;}}
\def\argmax{\mathop{\rm argmax}}
\def\argmin{\mathop{\rm argmin}}
\def\vec{\text{vec}}
\def\cov{\text{cov}}
\def\dg{\text{diag}}

\newcommand{\tabref}[1]{Table~\ref{#1}}
\newcommand{\secref}[1]{Sec.~\ref{#1}}
\newcommand{\figref}[1]{Fig.~\ref{#1}}
\newcommand{\lemref}[1]{Lemma~\ref{#1}}
\newcommand{\thmref}[1]{Theorem~\ref{#1}}
\newcommand{\clmref}[1]{Claim~\ref{#1}}
\newcommand{\crlref}[1]{Corollary~\ref{#1}}
\newcommand{\eqnref}[1]{Eqn.~\ref{#1}}

\newtheorem{remark}{Remark}
\newtheorem{theorem}{Theorem}
\newtheorem{lemma}{Lemma}
\newtheorem{definition}{Definition}

\newtheorem{proposition}{Proposition}

\title{Adversarial Attacks are Reversible with Natural Supervision}

\author{Chengzhi Mao$^1$, Mia Chiquier$^1$, Hao Wang$^2$, Junfeng Yang$^1$, Carl Vondrick$^1$\\
	$^1$Columbia University, $^2$Rutgers University\\
	{\tt\small \{mcz, mia.chiquier, junfeng, vondrick\}@cs.columbia.edu, hoguewang@gmail.com}}

\maketitle
\ificcvfinal\thispagestyle{empty}\fi

\begin{abstract}

We find that images contain intrinsic structure that enables the reversal of many adversarial attacks. Attack vectors cause not only image classifiers to fail, but also collaterally disrupt incidental structure in the image. We demonstrate that modifying the attacked image to restore the natural structure will reverse many types of attacks, providing a defense. Experiments demonstrate significantly improved robustness for several state-of-the-art models across the CIFAR-10, CIFAR-100, SVHN, and ImageNet datasets.
Our results show that our defense is still effective even if the attacker is aware of the defense mechanism. Since our defense is deployed during inference instead of training, it is compatible with pre-trained networks as well as most other defenses. 
Our results suggest deep networks are vulnerable to adversarial examples partly because their representations do not enforce the natural structure of images.

\end{abstract}

\section{Introduction}

Deep networks achieve strong performance over a number of computer vision tasks, yet they remain brittle under adversarial attacks \cite{AA, CW, mim, intriguing}.
With crafted perturbations, attackers can undermine predictions from the state-of-the-art models by changing the features in the representation \cite{TLA}.  These limitations prevent application of deep networks to sensitive and safety-critical applications \cite{Pei_2017, facenet, pointpillar, waymood}, underscoring the gap between current machine learning algorithms and human-level abilities~\cite{EvalAdvRob}.

A large body of work has studied how to \emph{train} deep networks such that they are robust to adversarial attacks. Adversarial training and its variants  \cite{madry, TLA, intriguing, pang2020bag, rice2020overfitting}, including multitask learning \cite{Mao2020MTR, self-supervise-adv-robust} and semi-supervised learning \cite{unlabeled}, significantly improve robustness. However, while existing methods focus on improving the training algorithm, they are burdened because they need to find a single representation that \emph{also} works for all possible corruptions and attacks. Training-based defenses cannot adapt to the individual characteristics of each attack at testing-time.

In this paper, we introduce an approach for \emph{reversing} the attack process, allowing us to formulate a defense strategy that adapts to each attack during the testing phase. Just as an attacker finds the right additive perturbation to \emph{break} the input, our approach will find the right additive perturbation to \emph{repair} the input. Figure \ref{fig:attack_vis} shows our reverse attack on a poisoned ImageNet image. However, reverse attacks are more challenging to produce than standard attacks because the category label is unknown to us during testing.

\begin{figure}[t]
\centering
\includegraphics[width=\linewidth]{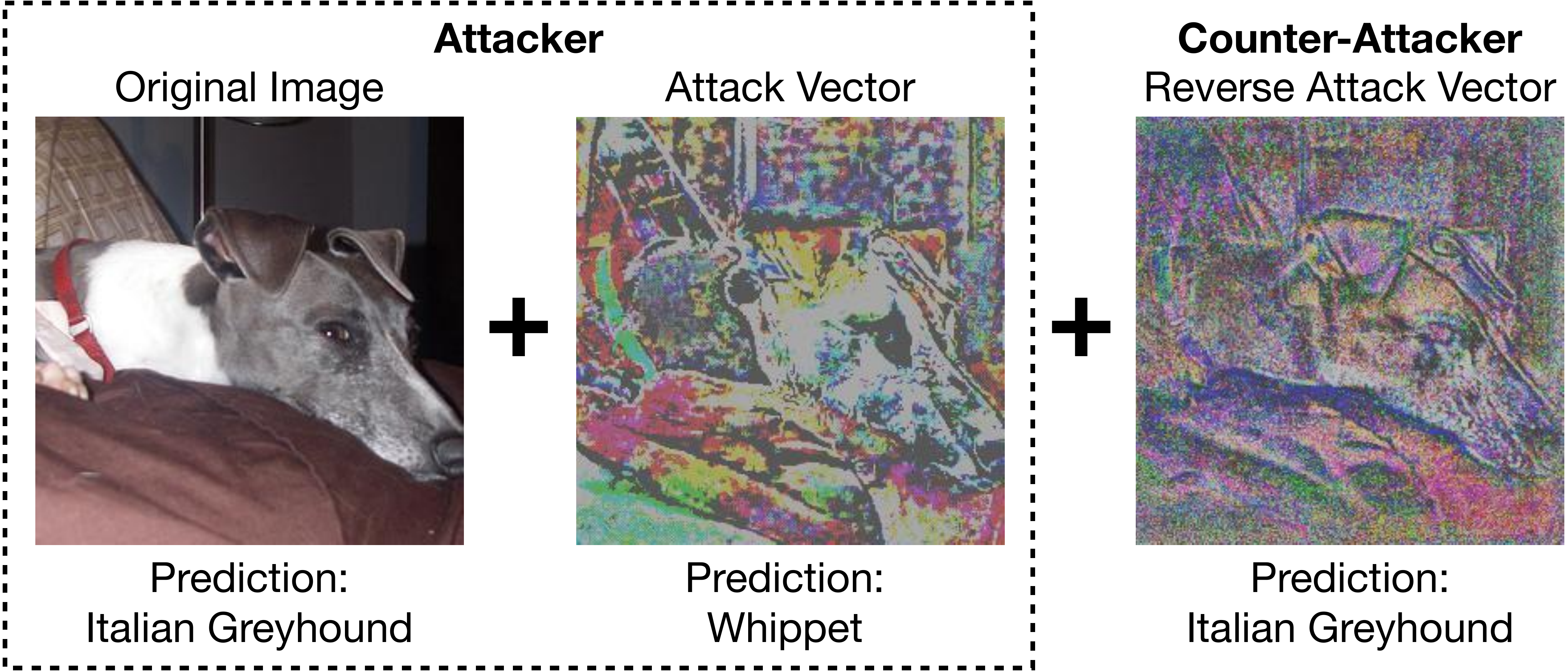}
  \caption{\textbf{Reverse Attacks:} Adversarial attacks are small perturbations that cause classification networks to fail \cite{madry, intriguing}. In this paper, we show there are intrinsic signals in natural images to reverse many types of attacks. In the right column, we visualize our reverse attack on an ImageNet image. Note that both attack vectors have been multiplied by ten for visualization purposes only.}
  \label{fig:attack_vis}
  \vspace{-1em}
\end{figure} 

Our key insight is that images contain natural and intrinsic structure that we can leverage to reverse many types of adversarial attacks. We found that, although adversarial attacks aim to fool the image classifier, they also collaterally damage self-supervised objectives. Our approach shows how to capitalize on this incidental signal in order to create adversarial defenses. By using self-supervision for defense at test time, we can guarantee that even the strongest adversary cannot manipulate the intrinsic signals that naturally come with the images, providing a more robust defense than training-based methods.

A key advantage of our framework is that it factors out the defense strategy from the visual representation.
Since reverse attacks are adaptive, this defense is able to efficiently scale to any corruption that violates the natural image manifold.  
Moreover, the modularity of our approach allows it to work with any classifier and complements existing defense models. It can also be integrated into future defense models and defend against novel attacks that corrupt natural image structures.

Visualizations, empirical experiments, and theoretical analysis show that
our reversal strategy significantly improves robust prediction for several established benchmarks and attacks. Our method advances the state-of-the-art defense methods by a large margin across four natural image datasets including CIFAR-10 (over 7.5\% gain), CIFAR-100 (over 5.5\% gain), SVHN (over 11.8\% gain), and ImageNet (over 3.0\% gain). Our method is robust against established attacks, including PGD \cite{madry}, C\&W \cite{CW}, and AutoAttack \cite{croce2020reliable}. In addition, our empirical results demonstrate that, even when the attacker is aware of our defense mechanism, our approach remains robust. 
Our models, data, and code are available at \url{https://github.com/cvlab-columbia/SelfSupDefense}.

\section{Related Work}
\begin{figure}
	\centering
	\includegraphics[width=\linewidth]{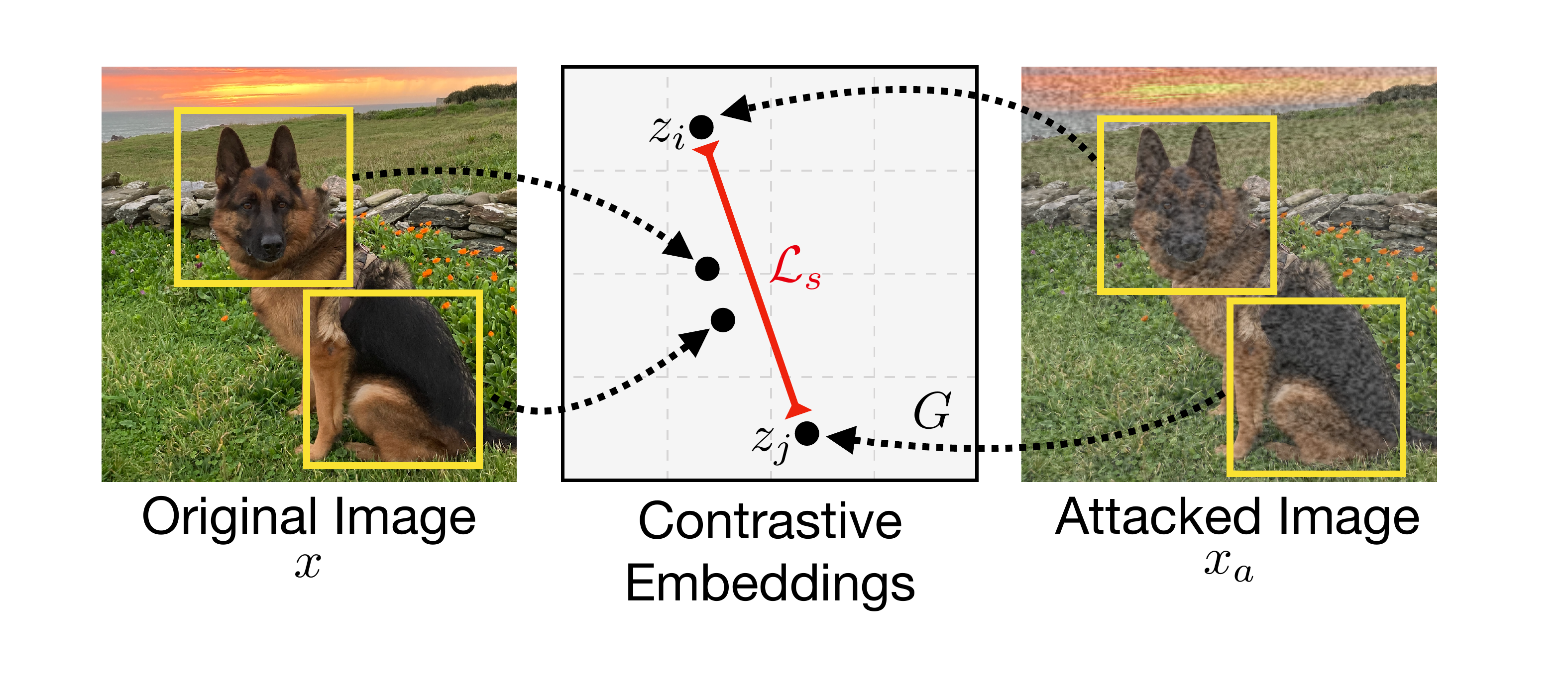}
	\vspace{-2em}
	\caption{\textbf{Defense Overview:} We find that adversarial attacks on classification networks will also collaterally attack self-supervised contrastive networks \cite{chen2020simple}. Since self-supervision is available during deployment, we exploit this discrepancy to reverse adversarial attacks and provide a defense. Our approach modifies the potentially attacked input image such that contrastive distances $\mathcal{L}_s$ are restored.}
	\vspace{-3mm}
	\label{fig:method}
\end{figure}

\textbf{Self-supervised Learning:} Natural images contain rich information for representation learning. Self-supervised learning enables us to learn high quality representations from images without annotations \cite{root_ssl, chen2020simple, zhang2016colorful, he2019moco, chen2020mocov2, SpeedNet, cluster_unsupervised, inpainting}.
By solving pretext tasks, such as jigsaw puzzles \cite{jigsaw}, image inpainting \cite{inpainting}, rotation prediction \cite{rotation}, image colorization \cite{zhang2016colorful, vondrick2018tracking}, random walk \cite{jabri2020walk},  and clustering \cite{caron2021unsupervisedcluster}, the learned representations can generalize to unseen downstream tasks such as image recognition~\cite{chen2020simple}, and also allow domain adaptation at testing time \cite{sun2020test}. Recently, contrastive learning has significantly advanced image recognition \cite{chen2020simple, he2019moco, PIRL, grill2020bootstrap}. In this paper, we leverage this incidental structure to correct adversarial attacks. Our defense uses the contrastive learning  task \cite{chen2020simple}, and it is extensible to existing self-supervised tasks as well \cite{inpainting, jigsaw, rotation}.

\begin{figure}[t]
\centering
\vspace{-1.5em}
\includegraphics[width=0.5\textwidth]{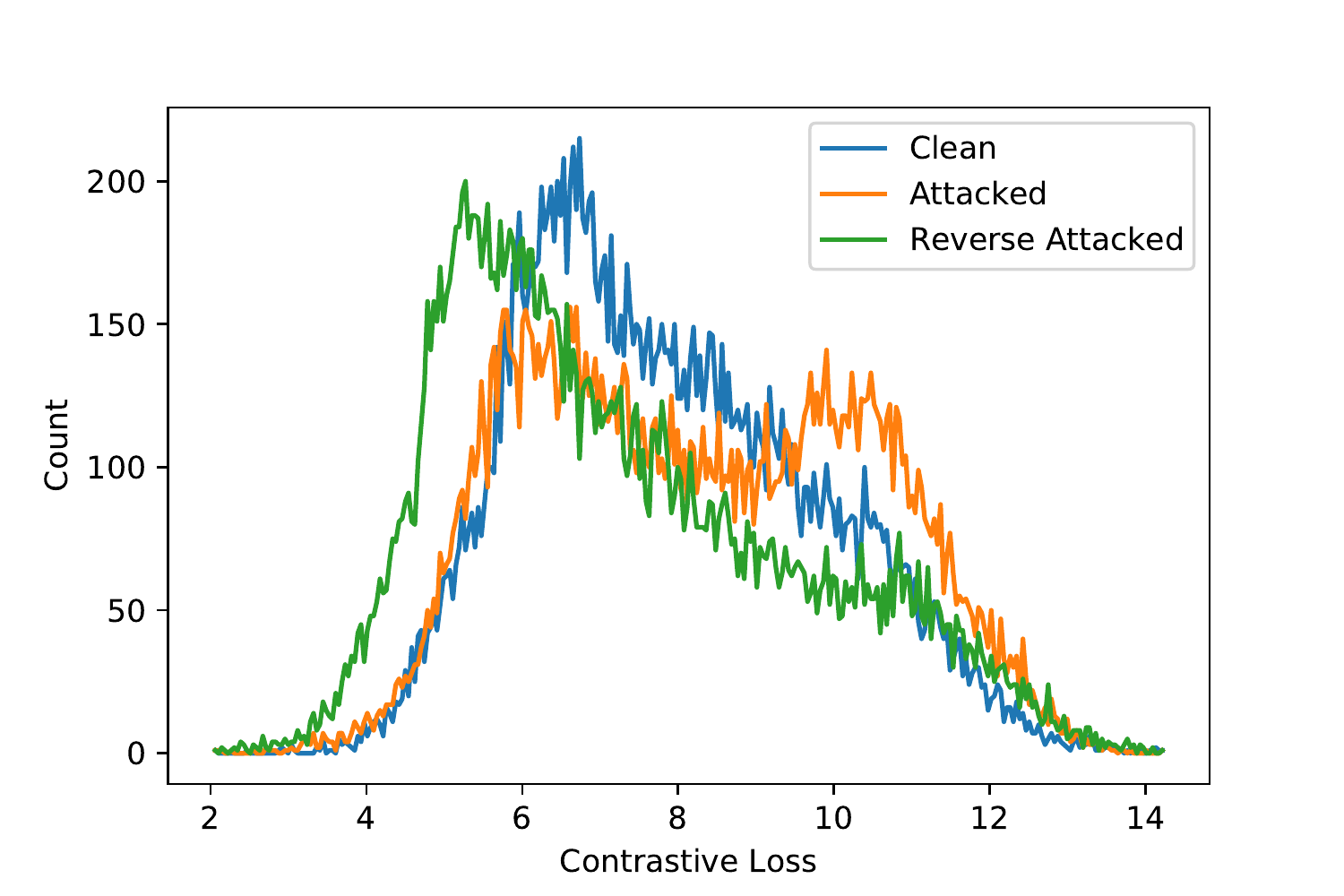}
\vspace{-1em}
  \caption{\textbf{Contrastive Score Distribution:} We histogram the constrastive loss value \cite{chen2020simple} for natural images (blue), after adversarial attack (orange), and after our reverse attack (green). This plot shows that adversarial attacks cause the contrastive loss to increase. We create a counter-attack by finding a perturbation that restores the self-supervised loss.}
  \vspace{-3mm}
  \label{fig:histogram}
\end{figure}

\textbf{Adversarial Robustness:} A large number of adversarial attacks have been proposed to fool deep models \cite{intriguing, obfuscated, CW, BIM, JSMA, moosavidezfooli2016deepfool}.  Special adversarial attacks that can be reversed to clean images are also proposed \cite{yin2021reversible}. Different from the existing approach to construct reversible attacks \cite{yin2021reversible}, our approach aims to reverse any unknown attacks for defense.  While many defense methods are proved to be not robust \cite{odds_odds, TurnWeak_fake, kwinner_fake, generative_rob_fake, fourier_fake, error_correcting_fake, ensemble_diversity_fake, transform-defense_fake, DefenseGAN_fake, SAP_fake, ther-encoding_fake} as they relied on gradient obfuscation, gradient masking \cite{obfuscated, EvalAdvRob}, and weak adaptive attack evaluation~\cite{adaptive_attacks}, adversarial training and its variants are proved to achieve the true robustness \cite{harnessing, madry, TLA, TRADES, rice2020overfitting, pang2020bag, wong2020fast, MART, AWP}. Moreover, recent progress shows that unlabeled data \cite{unlabeled, hendrycks2019pretraining} and self-supervised learning \cite{self-supervise-adv-robust} improve the robustness of deep models. While training a robust neural network to defense is vastly studied, no existing work investigates algorithms that improve robustness at inference time.

\section{Method}

We will first present a reverse attack that uses self-supervision at deployment time to defend against adversarial attacks. We then analyze the case where the attacker is aware of our defense, and show our defense remains effective. We finally provide theoretical justification for the robustness of our approach.

\subsection{Attacks and Reverse Attacks}

Let $\x$ be an input image, and $\y$ be its ground-truth category label.
To perform classification,  neural networks commonly learn to predict the category $\hat{\y}=F_{\theta}(\x) $
by optimizing the cross entropy  $H(\hat{\y}, \y)$ between the predictions and the ground truth.
The network parameters $\theta$ are estimated by minimizing the expected value of the objective:
\begin{equation}
    \mathcal{L}_c(\x, \y) = H \left(F_{\theta}(\x), \y\right),
\end{equation}
which can be optimized with gradient-based descent.

\textbf{The Attack:} In order to corrupt this model, the adversarial attack finds additive perturbations $\boldsymbol{\delta}$ to the image such that $\x_{a} = \x + \boldsymbol{\delta}$ is no longer classified correctly by the trained network $F_{\theta}$. 
Attackers create these worst-case images by maximizing the objective:
\begin{equation}
    \x_{a} = \argmax_{\x_{a}} \mathcal{L}_c(\x_{a}, \y), \quad \text{s.t.} \quad ||\x_{a}-\x||_q < \epsilon,
    \label{eq:advattack}
\end{equation}
where the $q$ norm bound of the perturbation $\boldsymbol{\delta}=\x_{a}-\x$ is less than $\epsilon$, which keeps the perturbation minimal.

\textbf{The Reverse:} 
We aim to defend against these attacks by reversing the attack process. Just as the attack finds an additive perturbation to \emph{break} the input, we will find an additive perturbation to \emph{repair} the input. However, we cannot simply flip Equation \ref{eq:advattack} from a maximization to a minimization because the category labels $\y$ are unknown at deployment.

The key observation is that self-supervised objectives are always available because they do not depend on the labels $\y$. While adversarial attacks aim to corrupt the classifier, they will also impact self-supervised representations, which is a signal we will leverage for reversal. Let $\mathcal{L}_s(\x)$ be a self-supervised objective on the input $\x$.  We create the reverse attack vector $\r$ by minimizing the objective:
\begin{equation}
    \r = \argmin_{\r} \mathcal{L}_s(\x_{a} + \r), \quad \text{s.t.} \quad ||\r||_q < \epsilon_v,
\end{equation}
where $\epsilon_v$ defines the bound of our reverse attack. The solution $\r$ will modify the adversarial image $\x_{a}$ such that it satisfies our choice of self-supervised objective.

After finding the optimal $\r$, robust prediction is straightforward. Our defense adds the resulting perturbation vector $\r$ to the input $\x_{a}$ before predicting the classification result with the normal network forward pass: $\hat{\y} =  F_{\theta}(\x_{a}+\r)$.

An advantage of reverse attacks is that, since they do not rely on offline adversarial training, the defense will generalize to unseen adversarial attacks. Moreover, our defense is able to fortify existing models without re-training.

\subsection{Natural Supervision for Defense}

While any self-supervised task~\cite{inpainting, jigsaw, rotation} can construct the loss $\mathcal{L}_s$, we use the contrastive loss as our natural supervision objective \cite{chen2020simple, chen2020mocov2}, which is a state-of-the-art self-supervised representation learning approach. The contrastive objective creates features that maximize the agreement between positive pairs of examples while minimizing the agreement between negative pairs of examples. Pairs are typically created with an augmentation strategy~\cite{chen2020simple}. In our case, when we receive a potentially adversarial image $\x$, we create the positive examples by sampling different augmentations from it to create multiple positive pairs. We create the negative pairs in a similar way, except applying augmentations to the randomly selected images.

Since these pairs are constructible at evaluation time, we create reverse attacks that minimize the term:
\begin{equation}
    \mathcal{L}_s(\x) = -\mathbb{E}_{i,j}\left[\y_{ij}^{(s)} 
    \log \frac{\exp(\mathrm{cos}(\z_i, \z_j)/\tau)}{\sum_{k}\exp(\mathrm{cos}(\z_i, \z_k)/\tau)}  
    \right],
\end{equation}
where $\z$ are the contrastive features. We use $\y_{ij}^{(s)}$ to indicate which pairs are positive and which are negative. This indicator satisfies $\y_{ij}^{(s)} = 1$ if and only if the examples $i$ and $j$ are both from $\x$, and $0$ otherwise. $\tau$ is a scalar hyper-parameter, and $\mathrm{cos}$ denotes cosine similarity.

Figure \ref{fig:histogram} shows that adversarial attacks on classification objectives also attack the contrastive objective $\mathcal{L}_s$, even though the attacker never explicitly optimizes for it. 
When there is an attack, $\mathcal{L}_s(\x_{a})$ will be larger than on clean images $\mathcal{L}_s(\x)$. This gap provides the signal for reverse attacks.

Figure \ref{fig:method} provides an overview of this defense mechanism, and 
Algorithm \ref{algorithm: SSLattack} summarizes our procedure.

\begin{algorithm}[t]
\caption{Self-supervised Reverse Attack}
\label{algorithm: SSLattack}
\begin{algorithmic}[1]
\STATE {\bfseries Input:} Potentially attacked image $\x$, step size $\eta$, number of iterations $K$, a classifier $F$, reverse attack bound $\epsilon_v$, and self-supervised loss function $\mathcal{L}_s$.
\STATE {\bfseries Output:} Class prediction $\hat{y}$
\STATE{\bfseries Inference:}

\STATE{$\x'\leftarrow \x+\n$, where $\n$ is the initial random noise}
\FOR{$k=1,...,K$}
\STATE{$\x'\leftarrow \x' - \eta \nabla_{\x'} \mathcal{L}_s(\x')$}
\STATE{$\x'\leftarrow \Pi_{(\x,\epsilon_v)} \x' $, which projects the image back into the bounded region.}
\ENDFOR
\STATE{Predict the final output by $\hat{y}=F(\x')$}

\end{algorithmic}
\end{algorithm}

\textbf{Contrastive Feature Estimation:}
To estimate the contrastive features $\z$,
we take the features before logits from a backbone $F$ and pass them to a two-layer network $G$. To compute the positive features, we sample augmentations conditioned on the input image $\x$. We follow a similar procedure to compute the contrastive features for the negative examples $\z_k$, sampling random images from a collection of images that form the negative set.

Offline, we fit the contrastive model $G$ on a large set of clean images using the same procedure as \cite{chen2020simple, chen2020mocov2}. We sequentially apply two augmentations: random cropping then scale back to the original size, and random color distortions including color jittering and random gray-scale. We found removing the Gaussian blur from the augmentations improved performance because it otherwise favored over-smooth perturbations. After $G$ is trained on clean images, we use it during reverse attacks without any further training.

\subsection{Analysis of Defense Aware Attack} 
\label{sec:defense_aware_attack}

In this section,
we analyze the effectiveness of our approach when the attacker is aware of our defense. 

\textbf{Attack Model:}
Let us assume the attacker knows the contrastive model parameters and our defense strategy. In this setting, the attacker can adversarially optimize against our defense with the following alternating optimization: 
\begin{align}
    \r &= \argmin_{\r} \mathcal{L}_s(\x+\r), \\
    \boldsymbol{\delta} &= \argmax_{\boldsymbol{\delta}} \mathcal{L}_c(\x+\r+\boldsymbol{\delta}, \y).
\end{align}
From the attacker's perspective, the above procedure is not ideal because it involves an alternating, min-max optimization. Past work suggests that this leads to unstable gradient estimation, having a gradient obfuscation problem that reduces the attack efficiency \cite{obfuscated}.

Similar to C\&W \cite{CW} and L-BFGS attack \cite{intriguing}, the attacker can reformulate the above equation as a constrained optimization problem:
\begin{equation}
\begin{aligned}
    \mathrm{maximize} \quad \mathcal{L}_c(\x_{a}, \y),    \quad \text{s.t.} \quad \mathcal{L}_s(\x_{a}) \leq \epsilon', \label{eq:opt}
\end{aligned}
\end{equation}
where $\epsilon'$ is the same value as the converged loss $\mathcal{L}_s$ for natural images. Intuitively, the attacker should  maximize the adversarial gain while respecting the self-supervised loss if they want to render our defense ineffective. 
    
To optimize Equation \ref{eq:opt}, they in practice maximize the following equation w.r.t. $\x_a$:
\begin{equation}
\label{eq:ada_attack}
    \mathcal{L}_l(\x_{a}, \y, \lambda_s)=\mathcal{L}_c(\x_{a}, \y) - \lambda_s \mathcal{L}_s(\x_{a}).
\end{equation}
We derive Equation \ref{eq:ada_attack} from Equation \ref{eq:opt} via the Lagrange Penalty method \cite{LagrangePenalty}, where $\mathcal{L}_l$ is the new loss for the adaptive attack. Full derivations are in the supplementary.

\textbf{Multi-objective Trade-off:}
The above derivation shows that the attacker can attempt to bypass our reversal by also minimizing $\mathcal{L}_s(\x_{a})$, so that attacks mimic the self-supervised features of clean examples. If the attacker produces examples that are as good as the clean examples in terms of $\mathcal{L}_s(\x_{a})$, our defense would not be able to reverse the attack by further decreasing the loss $\mathcal{L}_s(\x_{a} + \r)$. 

However, as the attacker must solve a multi-objective optimization, they must trade-off between the two objectives.
The scalar $\lambda_s$ controls how aggressively the attacker will corrupt the self-supervised model. The attacker's ideal adversarial attack will first optimize for the Pareto frontier by maximizing $\mathcal{L}_l(\x_{a}, \y, \lambda_s)$ for each $\lambda_s$.   They should select the $\lambda_s$ that yields the most damage (the lowest robust accuracy), and use the corresponding generated attack $\x_{a}^*$.

A larger $\lambda_s$ shifts the adversarial budget from attacking the classification loss $\mathcal{L}_c$ to attacking the self-supervised loss $\mathcal{L}_s$.
If the attacker is to attack the self-supervised task, they would then reduce the effectiveness of their classification attack, undermining their goal.
Attacking both $\mathcal{L}_s$ and $\mathcal{L}_c$  jointly requires creating adversarial images for multiple objectives, which is fundamentally more challenging  \cite{Mao2020MTR}.

Our defense creates a lose-lose situation for the attacker. If they ignore our defense, then we improve accuracy. If they account for our defense, then they hurt their attack.

\subsection{Theoretical Analysis}
\label{sec:theory}

We will show theoretical insights for why leveraging natural supervision improves adversarial robustness. Without our defense, the model predicts the category on an image with an incorrect estimate for the self-supervision label. With our defense, the model uses an image for which the self-supervision label is estimated correctly. We prove this increases the upper bound of the prediction accuracy.

A feed forward pass is equivalent to also including a latent self-supervised label to the model, since the information which the self-supervised network uses is in the image itself. We denote the ground-truth label of self-supervision as $\y^{(s)}$ and the predicted label of self-supervision under attack as $\y^{(s)}_a$. To make this latent label explicit in notation, we rewrite the loss functions as: $\mathcal{L}_s(\x) \rightarrow \mathcal{L}_s(\x, \y^{(s)})$ and $\mathcal{L}_s(\x_a)\rightarrow\mathcal{L}_s(\x_a, \y^{(s)}_a)$.  


\begin{lemma}
The standard classifier under adversarial attack is equivalent to predicting with $ P(\Y | \X=\x_a, \Y^{(s)} = \y^{(s)}_a)$, and our approach is equivalent to predicting with $P(\Y | \X=\x_a, \Y^{(s)} = \y^{(s)})$.
\end{lemma}

\begin{proof}
For the standard classifier under attack, we know that $P(\Y^{(s)}=\y^{(s)}_a|\X=\x_a) = 1$. Thus we know the standard classifier under adversarial attack is equivalent to
\begingroup\makeatletter\def\f@size{9}\check@mathfonts
\begin{align*}
P(\Y|\X=\x_a) &= \sum_{\Y^{(s)}} P(\Y^{(s)} | \X=\x_a) P(\Y|\Y^{(s)},\X=x_a)\\
&=P(\Y|\Y^{(s)}=\y^{(s)}_a,\X=\x_a).
\end{align*}
\endgroup
Our algorithm finds a new input image $\x^{(n)}_{\mathrm{max}}$ that
\begingroup\makeatletter\def\f@size{9}\check@mathfonts
\begin{align*}
     &\argmax_{\x^{(n)}} P(\X^{(n)}=\x^{(n)}|\X=\x_a)P(\Y^{(s)}=\y^{(s)} |\X^{(n)}=\x^{(n)}) \\
     &= \argmax_{\x^{(n)}} P(\X^{(x)}=\x^{(n)}| \X=\x_a, \Y^{(s)}=\y^{(s)}).
\end{align*}
\endgroup

Our algorithm first estimate $\x^{(n)}_{\mathrm{max}}$ with adversarial image $\x_a$ and self-supervised label $\y^{(s)}$. We then predict the label $\Y$ using our new image $\x^{(n)}_{\mathrm{max}}$. Thus, our approach in fact
 estimates $P(\Y|\X^{(n)} = \x^{(n)}_{\mathrm{max}})P(\X^{(n)} =  \x^{(n)}_{\mathrm{max}} |\X=\x_a,  \Y^{(s)} = \y^{(s)})$. Note the following holds:
\begingroup\makeatletter\def\f@size{9}\check@mathfonts
\begin{align}
&P(\Y | \X=\x_a, \Y^{(s)} = \y^{(s)})\\
&= \sum_{\x^{(n)}} P(\Y|\x^{(n)})P(\x^{(n)}|\X=\x_a,  \Y^{(s)} = \y^{(s)}) \\
&\approx P(\Y|\X^{(n)} = \x^{(n)}_{\mathrm{max}})P(\X^{(n)} =  \x^{(n)}_{\mathrm{max}} |\X=\x_a,  \Y^{(s)} = \y^{(s)})
\end{align}
\endgroup
Thus our approach is equivalent to estimating $P(\Y | \X=\x_a, \Y^{(s)} = \y^{(s)})$.
\end{proof}

We use the maximum a posteriori (MAP) estimation $\x^{(n)}_{\mathrm{max}}$ to approximate the sum over $\X^{(n)}$ because: (1) sampling a large number of $\X^{(n)}$ is computationally expensive; (2) our results in Figure~\ref{fig:trade} shows that random sampling is ineffective; (3) our MAP estimate naturally produces a denoised image that can be useful for other downstream tasks.

Next we provide theoretical guarantees that our approach can strictly improve the bound for classification accuracy in~\thmref{thm:bound}. For convenience we introduce an additional random variable $\X_a$ representing the adversarial image.

\begin{theorem}\label{thm:bound}
Assume the base classifier operates better than chance and instances in the dataset are uniformly distributed over $n$ categories. Let the prediction accuracy bounds be $P(\Y|\Y^{(s)}_a, \X_a) \in [b_1, c_1]$ and $P(\Y|\Y^{(s)}, \X_a) \in [b_2, c_2]$. If the conditional mutual information $I(\Y;\Y^{(s)}| \X_a) > 0$, we have $b_2 \geq b_1$ and $c_2 > c_1$, which means our approach strictly improves the bound for classification accuracy.
\end{theorem}

\begin{proof}
If $I(\Y;\Y^{(s)}| \X=\x_a) > 0$, then it is straight-forward that:
\begin{align*}
    I(\Y;\Y^{(s)}, \X_a) > I(\Y;\Y^{(s)}_a, \X_a) = I(\Y; \X_a).
\end{align*}

We define $H(\epsilon_p)=-\epsilon_p \log \epsilon_p-(1-\epsilon_p)\log(1-\epsilon_p)$. Using the \emph{Fano's Inequality} \cite{fano} and the fact that $Q(\epsilon_p) = H(\epsilon_p) + \epsilon_p \log(n-1)$ is a monotonically increasing function when error rate $\epsilon_p < 1-\frac{1}{n}$, i.e., accuracy higher than random guessing,\footnote{The validity of this fact are explained in the supplementary.} we derive the upper bound of accuracy $c_1$ and $c_2$ to be:

\begin{align*}
1- \epsilon_p &\leq c_1 =  1 - Q^{-1}(- I(\Y;\X_a) + H(\Y)), \\
 1- \epsilon_p &\leq  c_2 =  1 - Q^{-1}(- I(\Y;\Y^{(s)}, \X_a) + H(\Y)), 
\end{align*}
where the upper bound is a function of the mutual information. Since $H(\Y)$ is a constant, a larger mutual information will strictly increase the bound. Detailed proof is in the supplementary material.
\end{proof}

Intuitively, the adversarial attack $\x_a$ will corrupt some mutual information between the label $\Y$ and natural structure $\Y^{(s)}$. Thus, there is additional mutual information between $\Y$ and $\Y^{(s)}$ given $\x_a$, i.e., $I(\Y;\Y^{(s)}| \X=\x_a) > 0$. \thmref{thm:bound} shows that by restoring information from the correct $\Y^{(s)}$, the prediction accuracy can be improved.

Theoretically, by optimizing the self-supervision loss, the defense aware attack is in fact predicting classification label given the right self-supervision label $\Y^{(s)}$, i.e., $P(\Y | \X=\x_a, \Y^{(s)} = \y^{(s)})$. According to our theory, the robust accuracy should increase due to the restored information. Overall, our defense is robust even under a defense aware adversary.

\begin{figure}[t]
\centering
\includegraphics[width=0.45\textwidth]{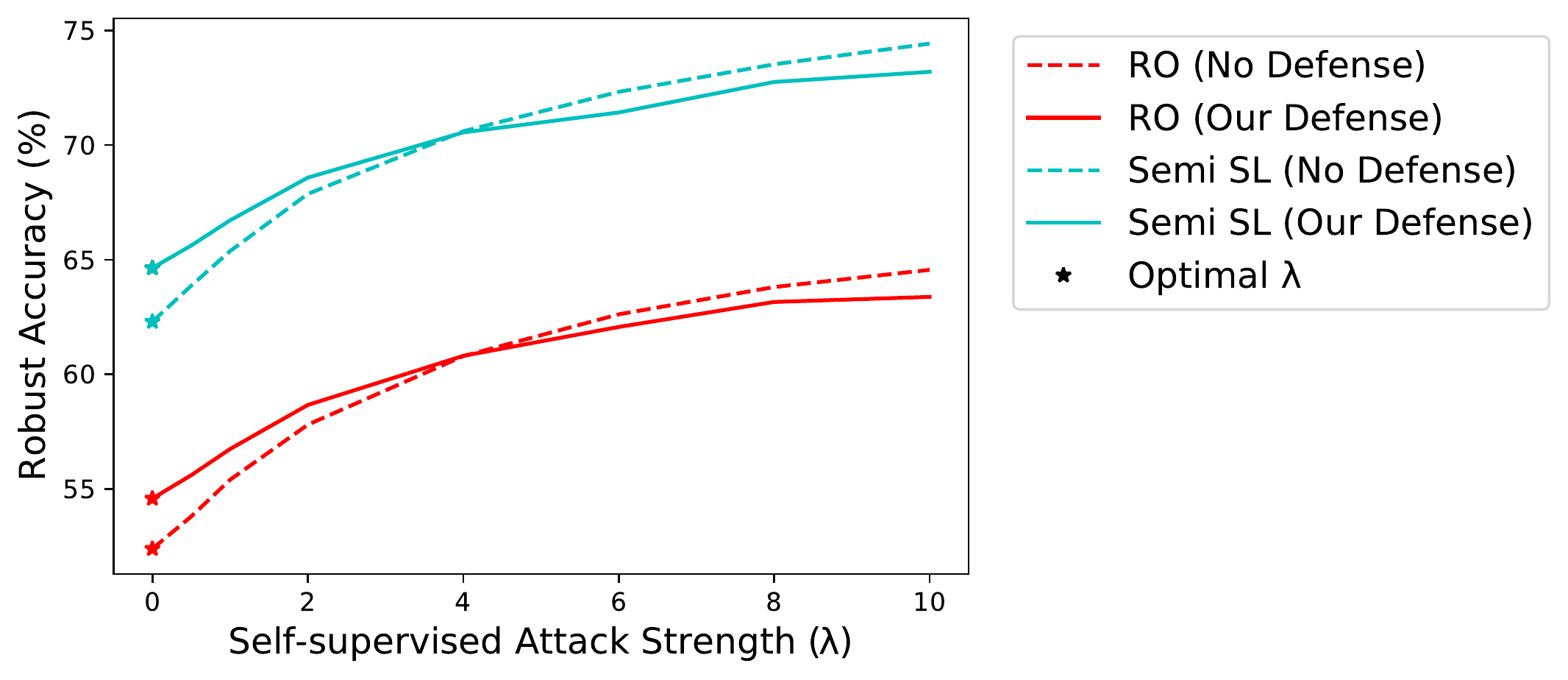}
  \caption{\textbf{The Trade-off:} The robust accuracy for defense aware attack under different $\lambda_s$ setup. We increase the $\lambda_s$ of defense aware attack from 0 to 10, and show robustness accuracy on two robust models, RO and Semi-SL, on CIFAR-10 dataset. The plot shows a trade-off. While increasing the value of $\lambda_s$ decreases the gain of Ours compared with Baseline, it also decreases the attacker's effectiveness. To achieve the best attack effectiveness, the attacker should use $\lambda_s=0$, which is the standard attack without attempting to corrupt the self-supervised task.}
  \label{fig:attackstrength}
  \vspace{-4mm}
\end{figure}

\begin{table*}[t]
	\begin{center}
		\vspace{-3mm}
		\scriptsize
		\centering
		\begin{tabular}{l|c|cc|cc|cc|cc|cc}
			\toprule
			& Model & \multicolumn{10}{c}{Adversarial Attack $L_{\infty}=8/255$ }  \\
			& Architecture & \multicolumn{2}{c|}{PGD 50}  & \multicolumn{2}{c|}{PGD 200} & \multicolumn{2}{c|}{BIM 200} & \multicolumn{2}{c|}{C\&W 200} & \multicolumn{2}{c}{AutoAttack} \\
			Inference Type & & Standard   & Ours  & Standard   & Ours & Standard   & Ours  & Standard   & Ours  & Standard   & Ours\\
			\midrule  
			
			TRADES \cite{TRADES} & WRN-34-10 & 55.05\% & \textbf{57.00\%} & 55.02\% & \textbf{57.18\%} & 55.06\% & \textbf{57.33\%} & 53.72\% & \textbf{56.56\%} & 53.16\% & \textbf{60.67\%} \\
			RO \cite{rice2020overfitting} & PreRes-18 & 52.40\% & \textbf{54.59\%} & 52.34\% & \textbf{54.62\%} & 52.32\% & \textbf{54.54\%} &  50.33\% & \textbf{53.48\%} & 48.95\% & \textbf{58.20\%}  \\
			BagT \cite{pang2020bag} & WRN-34-10 & 56.44\% & \textbf{58.33\%} & 56.40\% & \textbf{58.47\%} & 56.38\% & \textbf{58.55\%} & 54.82\% & \textbf{57.35\%} & 54.26\% & \textbf{61.70\%} \\
			MART \cite{MART} & WRN-28-10 & 62.72\% & \textbf{64.40\%} & 62.63\% & \textbf{64.26\%} & 62.54\% & \textbf{64.28\%} & 58.96\% & \textbf{62.18\%} & 57.29\% & \textbf{66.22\%}\\
			AWP \cite{AWP} & WRN-28-10 & 63.67\% & \textbf{64.21\%} & 63.64 \% & \textbf{64.07\%} & 63.64\% & \textbf{63.69\%} & 60.82\% & \textbf{61.90\%} & 60.58\% & \textbf{67.15\%} \\
			Semi-SL \cite{unlabeled} & WRN-28-10 &  62.30\% & \textbf{64.64\%} & 62.22\% & \textbf{64.44\%} & 62.18\% & \textbf{64.68\%} & 60.90\% & \textbf{63.83\%} & 60.22\% & \fbox{\textbf{67.79\%}} \\

			\bottomrule
		\end{tabular}
	\end{center}
	\vspace{-5mm}
	\caption{\small{Adversarial robust accuracy on the CIFAR-10 test set. Our method improves robustness of established work across different adversarial attack setups by over \textbf{7.5\%}. Our method (boxed) obtained \textbf{the first place on the CIFAR-10 AutoAttack leaderboard.}}} \label{tab:cifar10}
\end{table*}

\begin{table*}[t]
	\begin{center}
		\vspace{-3mm}
		\scriptsize
		\centering
		\begin{tabular}{l|c|cc|cc|cc|cc|cc}
			\toprule
			& Model& \multicolumn{10}{c}{Adversarial Attack $L_{\infty}=8/255$}  \\
			& Architecture &\multicolumn{2}{c|}{PGD 50}  & \multicolumn{2}{c|}{PGD 200} & \multicolumn{2}{c|}{BIM 200} & \multicolumn{2}{c|}{C\&W 200} & \multicolumn{2}{c}{AutoAttack}  \\
			Inference Type && Standard   & Ours  & Standard   & Ours & Standard   & Ours  & Standard   & Ours & Standard   & Ours\\
			\midrule  
			RO \cite{rice2020overfitting} & ResNet18 & 21.90\% & \textbf{23.83\%} & 20.90\% & \textbf{22.23\%} & 21.00\% & \textbf{22.09\%} & 20.42\% & \textbf{22.19\%} & 19.23\% & \textbf{25.45\%}\\
			TRADES* \cite{TRADES} & WRN34-10 & 27.04\% & \textbf{28.55\%} & 26.94\% & \textbf{28.16\%} & 26.94\% & \textbf{28.18\%} & 26.57\% & \textbf{27.88\% } & 25.55\% &\textbf{31.26\%}\\
			BagT * \cite{pang2020bag} & WRN34-10 & 29.87\% & \textbf{31.21\%} & 28.81\% & \textbf{31.15\%} & 29.81\% & \textbf{31.44\%} & 29.72\% & \textbf{31.40\%} & 27.59\% & \textbf{33.16\%} \\
			\bottomrule
		\end{tabular}
	\end{center}
	\vspace{-4mm}
	\caption{\small{Adversarial robust accuracy on the CIFAR-100 test set. Our method consistently improves robustness by over \textbf{5.5\%}.}} \label{tab:cifar100}
	\vspace{-5mm}
\end{table*}

\begin{table*}[h]
	\begin{center}
		\vspace{-3mm}
		\scriptsize
		\centering
		\begin{tabular}{l|c|cc|cc|cc|cc|cc}
			\toprule
			& Model& \multicolumn{10}{c}{Adversarial Attack $L_{\infty}=8/255$}  \\
			& Architecture & \multicolumn{2}{c|}{PGD 50}  & \multicolumn{2}{c|}{PGD 200} & \multicolumn{2}{c|}{BIM 200} & \multicolumn{2}{c|}{C\&W 200} & \multicolumn{2}{c}{AutoAttack}\\
			Inference Type & & Standard   & Ours  & Standard   & Ours & Standard   & Ours  & Standard   & Ours & Standard   & Ours\\
			\midrule  
			RO \cite{rice2020overfitting} & PreRes-18 & 51.03\% & \textbf{ 53.21\%} & 50.82\% & \textbf{53.06\%} & 50.84\% & \textbf{53.26\% } & 48.90\% & \textbf{54.62\%} & 47.15\% & \textbf{57.79\%}\\
			Semi-SL \cite{unlabeled} & WRN-28-10 & 55.69\% & \textbf{62.40\%} & 55.29\% & \textbf{62.12\%} & 55.43\% & \textbf{62.53\%} & 57.03\% & \textbf{63.34\%} & 53.61\% & \textbf{65.50\%}\\
			\bottomrule
		\end{tabular}
	\end{center}
	\vspace{-5mm}
	\caption{\small{Adversarial robust accuracy on the SVHN test set under different attacks. Our method improves robustness by over \textbf{11.8\%}.} }\label{tab:svhn}
\end{table*}

\begin{table}[h]
\begin{center}
      \vspace{-3mm}
    \centering
    \scriptsize
    \begin{tabular}{l|c|c|c|c|c}
         \toprule
         & \multicolumn{5}{c}{Adversarial Attack $L_{\infty}=4/255$}  \\
         & FGSM & PGD10 & PGD 20 & BIM 20 & C\&W 20  \\
          
         \midrule  
          FBF \cite{wong2020fast} & 32.47\%  & 28.68\% & 28.52\% & 28.49\%  & 27.81\% \\
          Ours + FBF & \textbf{33.51\%}  & \textbf{31.36\%} & \textbf{31.32\%} & \textbf{31.27\%}  & \textbf{30.88\%} \\
         \bottomrule
    \end{tabular}
\end{center}
      \vspace{-2mm}
\caption{\small{Adversarial robust accuracy on the ImageNet test set. Our method uses the same model and parameters as the baseline FBF \cite{wong2020fast}. After reversing the adversarial attack with our natural supervison, we improve robustness on ImageNet by over \textbf{3.07\%}.}} \label{tab:imagenet}
\end{table}

\begin{table}[h]
\begin{center}
      \vspace{-3mm}
    \scriptsize
    \centering
    \begin{tabular}{l|cc|cc}
         \toprule
         &  \multicolumn{4}{c}{Adversarial Attack $L_2=256/255$}  \\
         &  \multicolumn{2}{c|}{PGD 200}  & \multicolumn{2}{c}{C\&W 200} \\
          Inference Type & Standard   & Ours & Standard   & Ours  \\
         \midrule  
        TRADES \cite{TRADES} & 36.90\% & \textbf{39.16\%} & 35.89\% & \textbf{38.07\%} \\
          RO \cite{rice2020overfitting} & 40.00\% & \textbf{41.98\%} & 38.31\% & \textbf{40.17\%} \\
          BagT \cite{pang2020bag} & 38.80\% & \textbf{41.28\%} & 37.01\% & \textbf{39.20\%} \\
          Semi-SL \cite{unlabeled} & 42.07\% & \textbf{44.47\%} & 39.97\% & \textbf{42.72\%} \\
          AWP \cite{AWP} & 44.39\% & \textbf{47.15\%} & 40.97\% & \textbf{43.72\% }\\
          MART \cite{MART} & 43.24\% & \textbf{45.85\%} & \fbox{43.23\%} & \fbox{\textbf{45.85\%} }\\
         \bottomrule
    \end{tabular}
\end{center}
      \vspace{-4mm}
\caption{\small{$L_2$ norm bounded adversarial robust accuracy on the CIFAR-10. Our natural supervision is agnostic to the attack type and can improve the robustness by over 2.6\% for $L_2$ attack without retraining the defense model. The lower bound for the best achieved robustness on $L_2$ is \fbox{boxed}.}} \label{tab:l2cifar10}
\vspace{-5mm}
\end{table}

\section{Experiments}
Our experiments evaluate the robustness at image classification on four datasets: CIFAR-10 \cite{cifar10}, CIFAR-100 \cite{cifar100},  SVHN \cite{SVHN}, and ImageNet \cite{imagenet_cvpr09}. We compare with the state-of-the-art defense methods, under several strong adversarial attacks including a defense aware attack.

\subsection{Baselines}

We apply our method to seven established, scrutinized \cite{obfuscated} defense methods including the state-of-the-art adversarial robust model. All studied methods are trained with adversarial training \cite{madry}, but achieve higher robust accuracy than the initial version of Madry et al. \cite{madry}.

\textbf{TRADES \cite{TRADES}} is the winning solution for NeurIPS 2018 Adversarial Vision Challenge. It introduces a KL-divergence term to regularize the representation of adversarial examples to match the ones of clean examples.

\textbf{Robust Overfit (RO) \cite{rice2020overfitting}} re-examines the existing adversarial robust models through overfitting, which is the state-of-the-art model trained with Pre-ResNet18.

\textbf{Bag of Tricks (BagT) \cite{pang2020bag}} conducts extensive experiments on the effect of hyper-parameters on adversarial training \cite{madry}. It is the state-of-the-art adversarial robust model without additional unlabeled data for training. 

\textbf{Semi-supervised Learning (Semi-SL) \cite{unlabeled}} significantly improves adversarial robustness using unlabeled data. By training with pseudo labels of unlabeled images, the model achieved the state-of-the-art robustness. However, this work neglects the information of natural images beyond the pseudo classification label. 

\textbf{MART \cite{MART}} uses misclassification aware adversarial training to achieve improved robustness. We use its best version trained on top of Semi-SL \cite{unlabeled}.

\textbf{Adversarial Weight Perturbation (AWP) \cite{AWP}} trains robust model by smoothing the weights' loss landscape. We use its best version trained on top of Semi-SL \cite{unlabeled}.

\textbf{Fast is Better than Free (FBF) \cite{wong2020fast}} is the state-of-the-art solution for training robust ImageNet classifier in reasonable training budget and time. 

\subsection{Attack Methods}

\textbf{Fast Gradient Sign Method (FGSM) \cite{harnessing}} is a one-step adversarial attack to fool neural networks. 

\textbf{Projected Gradient Descent (PGD) \cite{madry}} is the standard evaluation for adversarial robustness, which optimizes the adversarial noise with gradient descent for iterations, and project the noise back to the nearest boundary if it is out of the given bound. 

\textbf{Basic Iterative Attack (BIM) \cite{BIM}} is a variant of PGD attack without the initial random start.

\textbf{C\&W Attack \cite{CW}} is a powerful iterative attack that has been widely used for robustness evaluation. It reduces the logit value for the right class while increasing that for the second best class to fool the classifier.

\textbf{AutoAttack \cite{croce2020reliable}} is the state-of-the-art attack, consisting of an ensemble of parameter-free adversarial attacks including auto-PGD \cite{AA}, FAB \cite{FBA}, and Square Attack \cite{ACFH2020square}.

\textbf{Defense Aware Attack} is discussed in Section \ref{sec:defense_aware_attack}, which is theoretically the optimal adaptive white-box attack to bypass our defense algorithm. 

\subsection{Experimental Settings}

\textbf{Backbone Architectures.} Following prior literature, we conduct experiments with Pre-ResNet18 \cite{preres}, ResNet50 \cite{ResNet}, and WideResNet \cite{WRN}. We download the pretrained models' weights online. \footnote{We reproduce a few models that are not available and denote by~*} 

\textbf{Self-supervised Learning Branch.} We use a network with two fully connected layers that takes in the features from the penultimate layer of the backbone network.

\textbf{Implementation Details.} We train our self-supervision model with the Adam \cite{kingma2017adam} optimizer. We use a learning rate of $0.001$. When training the self-supervised model, we use temperature $\tau=0.2$ for the contrastive loss, with a batch size of 128. For CIFAR-10 and CIFAR-100, we train the self-supervised branch for 200 epochs. For SVHN, we train it for 600 epochs. For ImageNet, we train for 30 epochs. We set the reverse attack bound to be $\epsilon_v=2\epsilon$ and optimization iterations to be $K=40$. We implement our model with Pytorch \cite{pytorch}. Please see supplementary for details.


\subsection{Results of Defense Aware Adversarial Attacks}
\label{sec:exp_res}

In Section \ref{sec:defense_aware_attack}, we discussed the strongest adaptive attack that can be used to bypass our defense. We show the results of the adaptive attacker on CIFAR-10 in Figure \ref{fig:attackstrength}. We use attacks with 50 steps, with perturbation bound $L_{\infty}=8/255$.  We vary the value of the $\lambda_s$ from 0 to 10, where 0 corresponds to the standard PGD attack without considering our defense strategy. The results show that increasing $\lambda_s$ and focusing more attack budget to the self-supervised defense, the gain of our approach is reduced (full line falls under dotted line). However, as $\lambda_s$ gets larger, the attack for classification task also gets weaker (line goes up). While adaptive attacks successfully reduce the additional gain bought by our approach, it too significantly sacrifices the initial attack success rate on the classification task. Minimizing the $\mathcal{L}_s(\x_{a}, \y^{(s)})$ hurts the original classification attack so much that it is not worth it for the attacker to account for our defense. We also show results with 500 steps in the supplementary, where our conclusion also holds.

This finding matches our initial theory in Section \ref{sec:theory}, where leveraging the incidental structure in the images improves robustness. The decrease of attack success rate on the target classification task is also consistent with prior work \cite{Mao2020MTR}, which suggests that it is harder to simultaneously attack multiple tasks at once. In fact, the attacker is trading off between classification attack success rate and fooling the self-supervised defense. For the attacker, the optimal attack is $\lambda_s=0$, which is the standard adversarial attack without considering our self-supervised defense. Therefore, we use this setup in the remaining experiments.

\subsection{Results with Optimal Adversarial Attacks}
 In the optimal setup, $\lambda_s=0$. The attack is equivalent to the standard adversarial attack without considering the defense branch. Thus the gain from \emph{Standard} to \emph{Ours} is the lower bound of our self-supervised correction. We now show the gain on four datasets.
 
 CIFAR-10 \cite{cifar10} contains 10 categories. In Table \ref{tab:cifar10}, we add our approach to six existing robust models, where we constantly correct by up to 7.5\% of the adversarial examples, as shown in the gain from Standard to Ours. 

CIFAR-100 \cite{cifar100} contains 100 categories. In Table \ref{tab:cifar100}, we show over 5.5\% gain compared with the baselines.

SVHN \cite{SVHN} is a 10 category street view house number dataset. In Table \ref{tab:svhn}, we experiment on two methods that have pretrained models available, including the state-of-the-art semi-supervised learning \cite{unlabeled}. Ours demonstrate over 11.8\% gain compared with the original defense method.

ImageNet \cite{imagenet_cvpr09} contains 1000 categories. We use the pretrained model from Wong et al. \cite{wong2020fast}, which is the state-of-the-art ResNet50 robust model. In Table \ref{tab:imagenet}, we use 5 different attacks to access the adversarial robustness of the original model and our model with natural supervision defense. Our approach achieved over 3\% gain on robustness. 

\begin{figure}[t]
\vspace{-8mm}
\centering
\subfloat[RO \cite{rice2020overfitting}]{\label{aasxs}\includegraphics[width=0.24\textwidth]{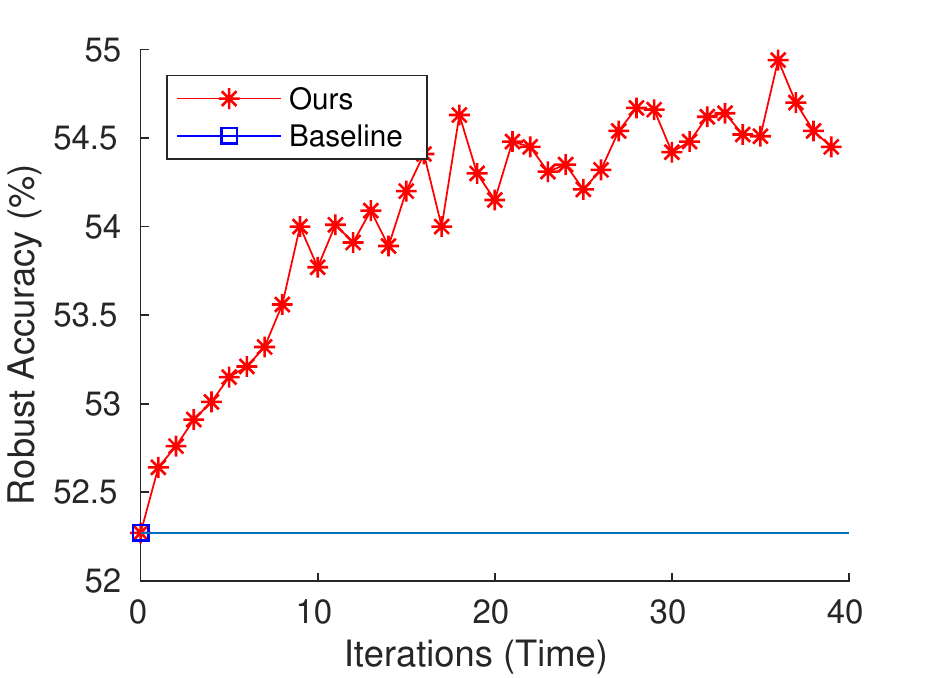}}
\subfloat[Semi-SL \cite{unlabeled}]{\label{fig:tsne_natural_two_clusters}\includegraphics[width=0.24\textwidth]{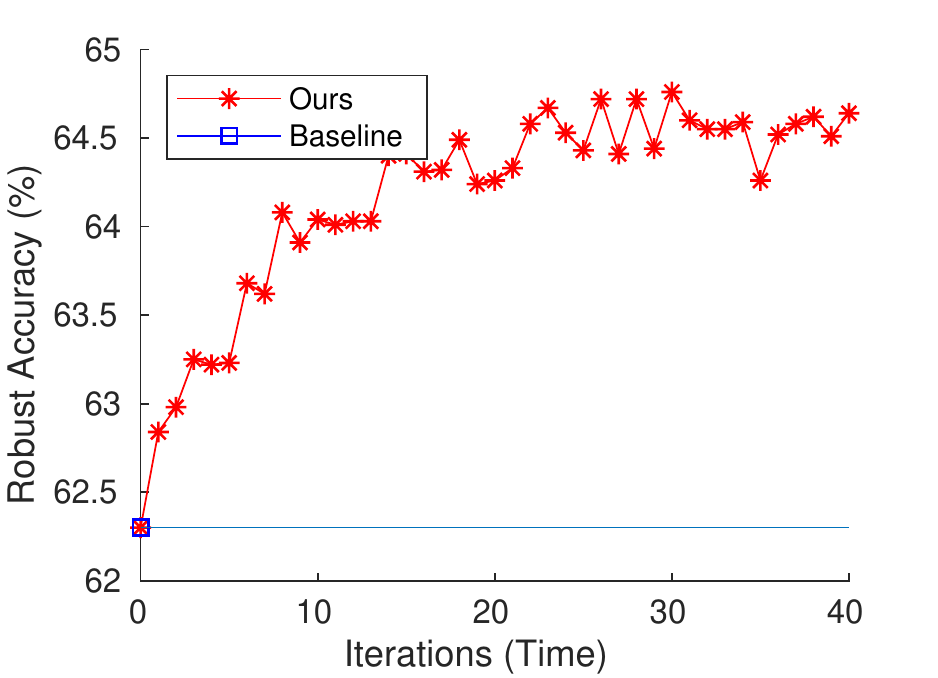}}
\caption{\textbf{Speed versus Robust Accuracy:} Trade-off between inference time and adversarial robustness on CIFAR-10 dataset. As our approach is iterative, we can stop early if the application prefers speed over the robustness. }
\label{fig:trade-time}
\vspace{-3mm}
\end{figure}

\begin{figure*}[h]
\vspace{-5mm}
\centering
\subfloat[CIFAR-10 (up to 9.2\% gain)]{\label{aasxs}\includegraphics[width=0.25\textwidth]{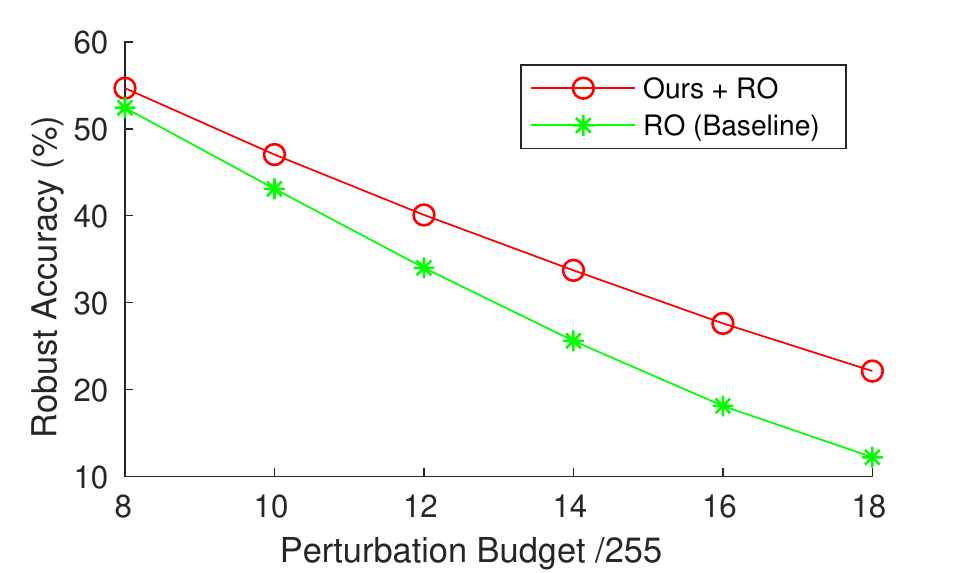}}
\subfloat[CIFAR-100 (up to 3.3\% gain)]{\label{fig:tsne_natural_two_clusters}\includegraphics[width=0.25\textwidth]{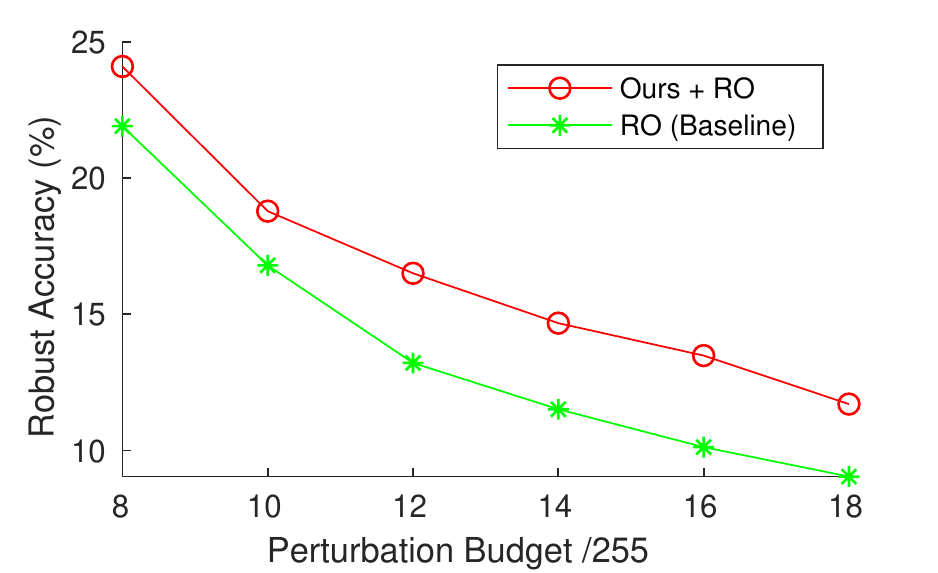}}
\subfloat[SVHN (up to 14.1\% gain)]{\label{fig:tsne_natural_two_clusters}\includegraphics[width=0.25\textwidth]{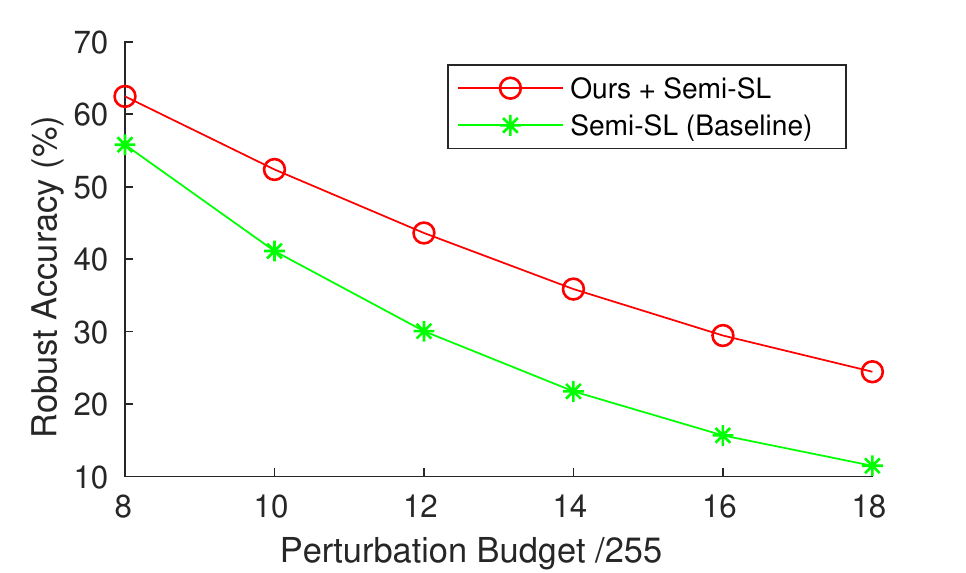}}
\subfloat[ImageNet (up to 4.4\% gain)]{\label{fig:tsne_natural_two_clusters}\includegraphics[width=0.25\textwidth]{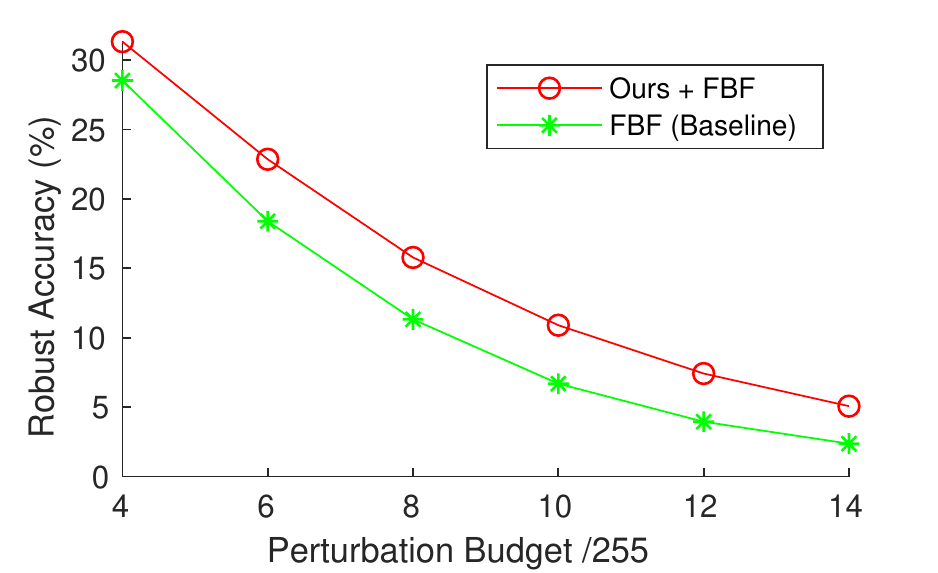}}
\vspace{-2mm}
\caption{\small{The adversarial robust accuracy vs. perturbation budget curves on CIFAR-10, CIFAR-100, SVHN, and ImageNet, under the $L_{\infty}$ norm. The red line is applying our inference algorithm to the baseline models \cite{rice2020overfitting, unlabeled, wong2020fast}. Using our inference algorithm significantly improves the robustness.}}
\label{fig:curve}
\vspace{-3mm}
\end{figure*}

\begin{figure*}[h]
\vspace{-2mm}
\centering
\subfloat[CIFAR-10]{\label{aasxs}\includegraphics[width=0.25\textwidth]{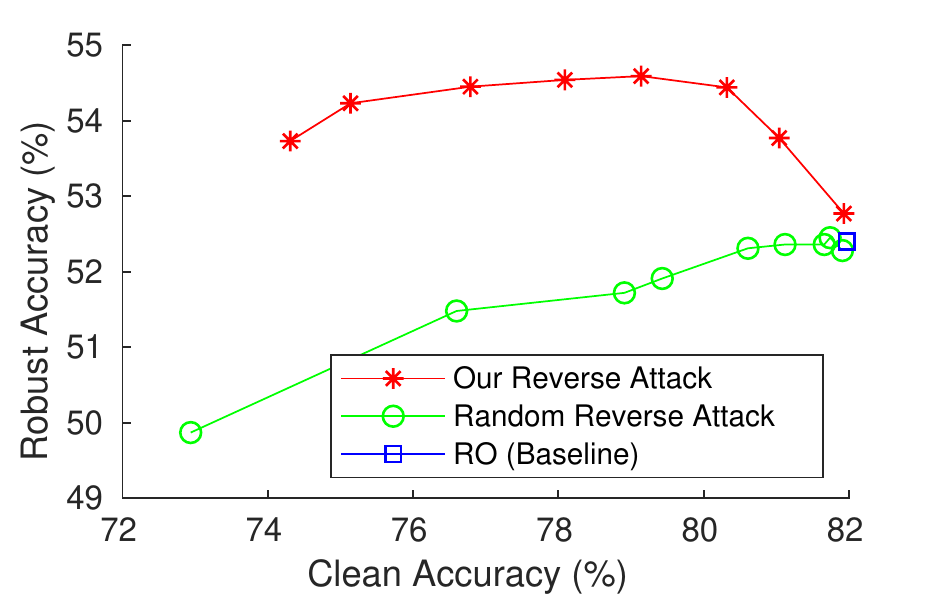}}
\subfloat[CIFAR-100]{\label{fig:tsne_natural_two_clusters}\includegraphics[width=0.25\textwidth]{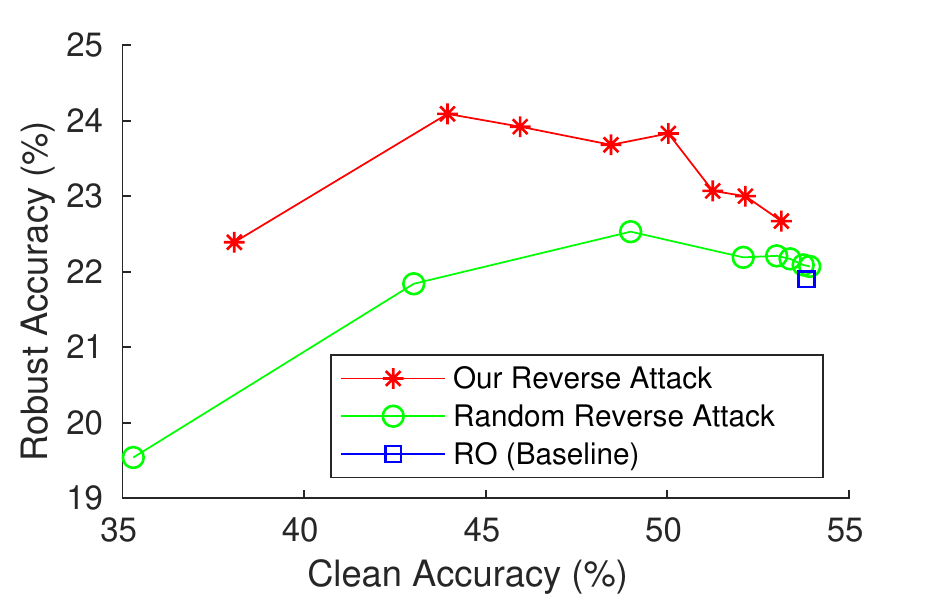}}
\subfloat[SVHN]{\label{fig:tsne_natural_two_clusters}\includegraphics[width=0.25\textwidth]{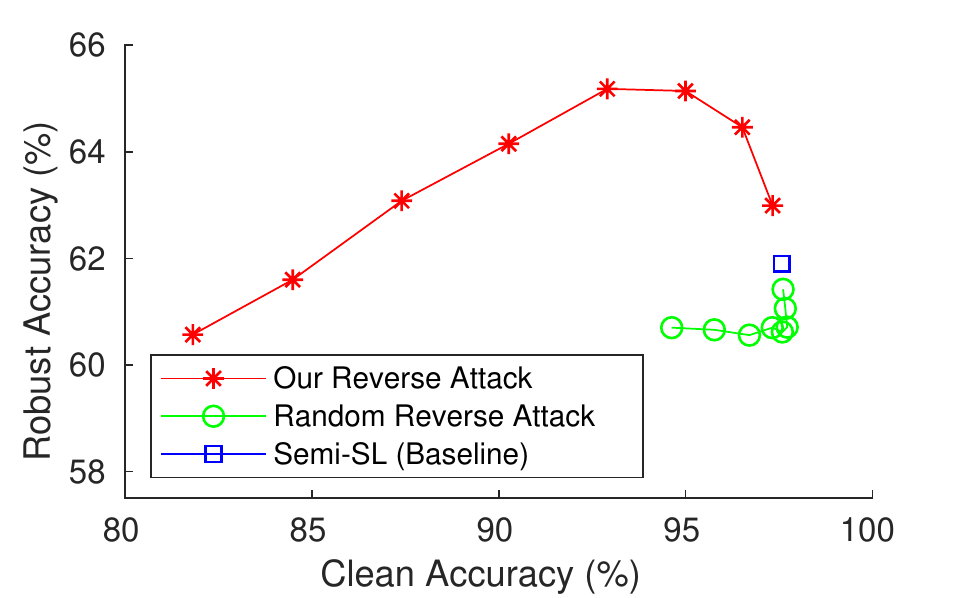}}
\subfloat[ImageNet]{\label{fig:tsne_natural_two_clusters}\includegraphics[width=0.25\textwidth]{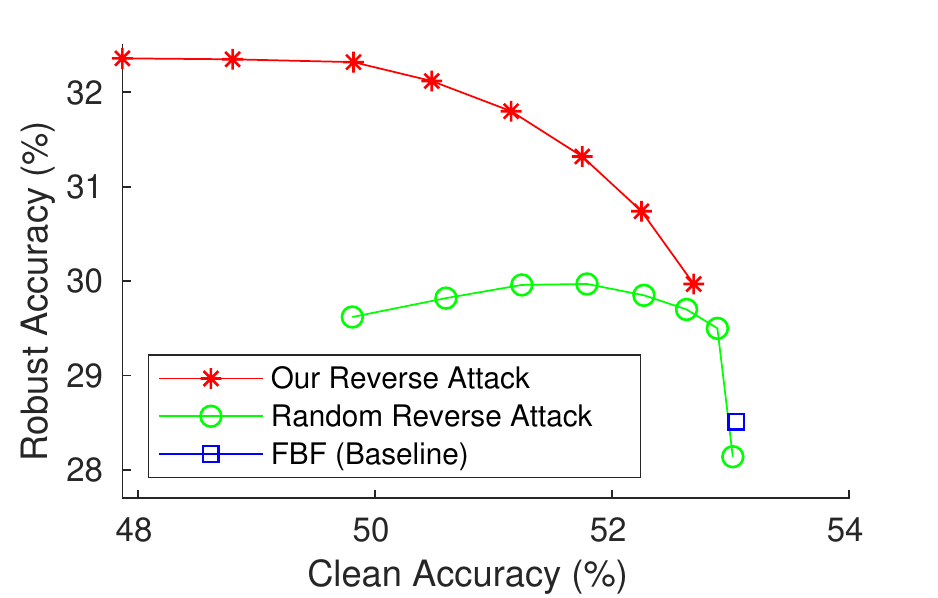}}
\vspace{-2mm}
\caption{\small{The trade-off between adversarial robust accuracy vs. clean accuracy on CIFAR-10, CIFAR-100, and SVHN under the $L_{\infty}$ norm. We increase the noise budget $\epsilon_v$ from small to large, which causes the clean accuracy to drop from right to left. Our method produces a better reversal of the adversarial perturbation than just adding random noise to reverse it.}}
\label{fig:trade}
\vspace{-3mm}
\end{figure*}

\subsection{Analysis}

\textbf{Accuracy vs. Time Budget.} As our method is iterative, we can adjust the number of iterations according to different time budget. We use the number of iterations conducted as a indicator of time, and plot the accuracy vs. time budget in Figure \ref{fig:trade-time}, where we can see even a few updates can significantly improves the robustness.

\textbf{Robustness Curve.} We adopt the robustness curve evaluation of adversarial robustness accuracy vs. the perturbation budget \cite{dong2020benchmarking}. We show the trend in Figure \ref{fig:curve}. We apply our inference algorithm as additional defense to existing robust models \cite{rice2020overfitting, wong2020fast}, where our approach achieved up to \textbf{14\%} robustness gain compared with standard inference method, especially when the attack gets stronger with larger perturbation bound.

\textbf{Trade-off between clean accuracy and adversarial robustness.} It has been proved that there exists a natural trade-off between clean accuracy and adversarial robustness given a classifier \cite{odds, TRADES}. In Figure \ref{fig:trade}, we compare our natural supervised reverse defense with the random reverse defense (baseline). We increase the additive noise level $\epsilon_v$ that is applied to reverse the adversarial examples as well as the clean examples (we use the same algorithm to clean examples because during inference we cannot distinguish adversarial ones from clean ones). The clean accuracy often drops as the noise level goes up. While there is a trade-off  between clean accuracy and robust accuracy \cite{odds, TRADES}, our approach achieved a better trade-off between them. 


\textbf{Robustness on $L_2$ norm bounded adversarial attacks.} We measure whether our defense can also generalize beyond the $L_{\infty}$ bounded attack. Table \ref{tab:l2cifar10} shows results under $L_2$ norm bounded attack on CIFAR-10, where our approach consistently improves robustness under $L_2$ norm bounded attacks by over 2.6\%.

\begin{figure}
\centering
\includegraphics[width=0.45\textwidth]{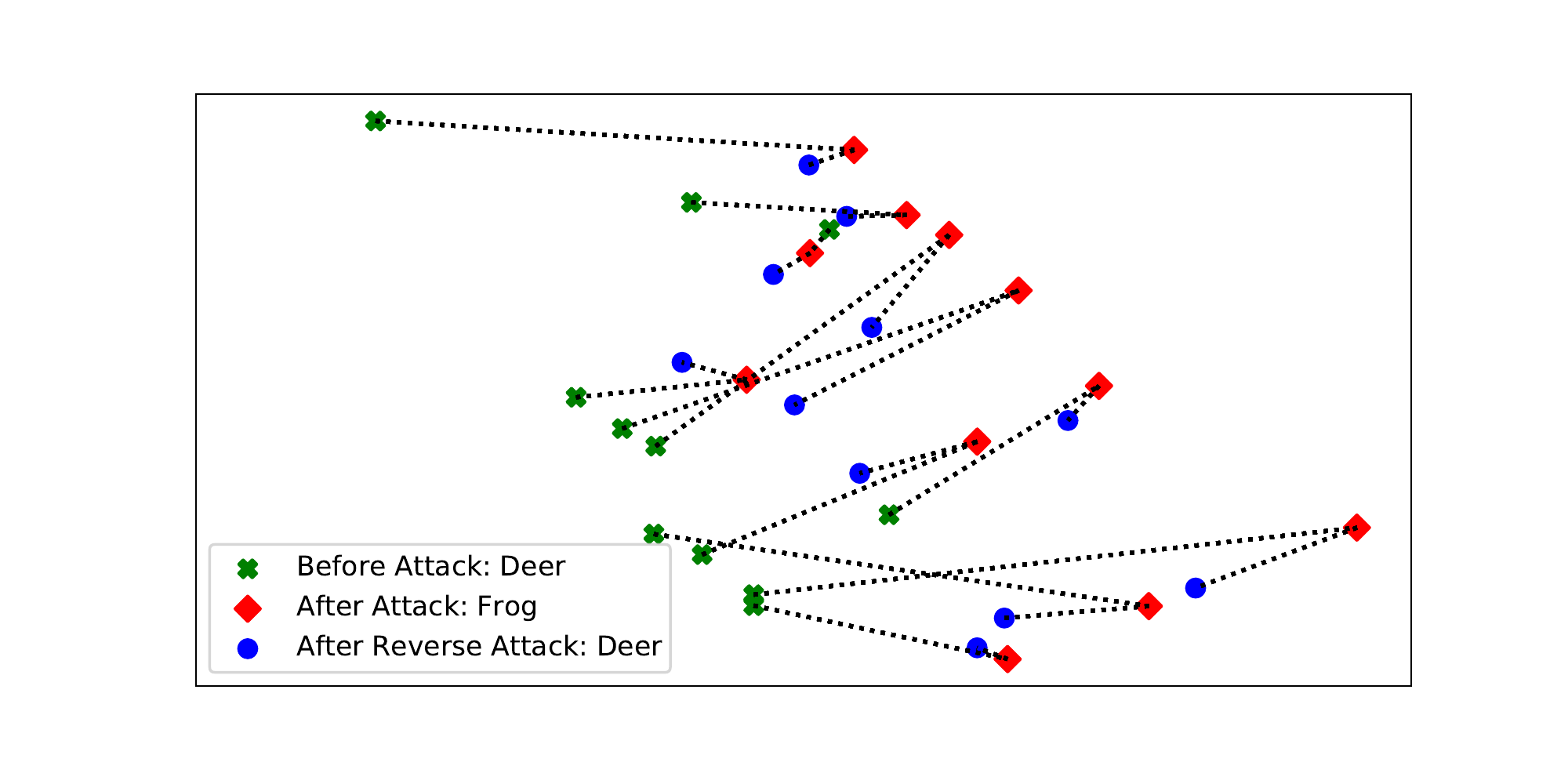}
  \caption{\textbf{Feature Trajectories:} We project the features onto a plane with PCA to visualize their trajectory under attack and our reverse attack. The green cross indicates the clean examples, the red square indicates the misclassified adversarial examples, and the blue dot indicates our reversal method.  Our approach pushes the misclassified examples (red) back to the original features (green), improving adversarial robustness.} 
  \label{fig:pca}
  \vspace{-5mm}
\end{figure} 

\textbf{Feature visualization.} Figure \ref{fig:pca} visualizes the trajectory of images' penultimate layer's feature as it transitions through an attack and the reversal. We use PCA to project the features onto a plane. The plot demonstrates that the attack shifts the feature embedding from the right class to the wrong class. Then, the reverse attack often returns the features back to the right class.

To quantify this effect, we take the Euclidean distance between the clean embedding and the attacked embedding, denoted $D_{c_a}$, as well as the Euclidean distance between the clean embedding and the inverse attacked embedding, denoted $D_{c_i}$, for the triples that have the same clean class and inverse attacked class. For all but one combination of categories,  $D_{c_a} >  D_{c_i}$.  Additionally, across all triplets, we checked how much the average distance from clean to reverse attacked is reduced from  the average distance from clean to attacked,  and obtained a value of roughly 10\% decrease. These results together demonstrate that, on average, our reverse attack returns the attacked embedding closer to the original embedding.

\section{Conclusions}

We introduce an approach to use natural supervision to reverse adversarial attacks on images. Our results demonstrate improved robustness across several benchmarks and several state-of-the-art attacks.
Our findings suggest integrating defense mechanisms into the inference algorithm is a promising direction to improve adversarial robustness.

\textbf{Acknowledgements:}  This research is based on work partially supported by NSF CRII \#1850069, NSF grant CNS-15-64055, ONR grants N00014-16-1-2263 and N00014-17-1-2788, a JP Morgan Faculty Research Award, and a DiDi Faculty Research Award. MC is supported by a CAIT Amazon PhD fellowship. We thank Nvidia for GPU donations. We thank Binghui Peng for discussions. The views and conclusions contained herein are those of the authors and should not be interpreted as necessarily representing the official policies, either expressed or implied, of the sponsors.

{\small
\bibliographystyle{ieee_fullname}
\bibliography{egbib}
}

\externaldocument[S-]{supplementary}

\newcommand{\theHalgorithm}{\arabic{algorithm}}

\date{}
\author{} 
\newpage
\newpage
\begin{subappendices}

\language0
\lefthyphenmin=2
\righthyphenmin=3

\maketitle




\section{Theoretical Results}

\subsection{Detailed Proof of Lemma 1}

\begin{lemma}
The standard classifier under adversarial attack is equivalent to predicting with $ P(Y | X=x_a, Y^{(s)} = y^{(s)}_a)$, and our approach is equivalent to predicting with $P(Y | X=x_a, Y^{(s)} = y^{(s)})$.
\end{lemma}

First we show that
\begin{align*}
P(Y | X = x_a) = P(Y | X=x_a, Y^{(s)} = y^{(s)_a}).
\end{align*}
It is easy to see that
\begin{align*}
P(Y|X=x_a) &= \sum_{Y^{(s)}} P(Y^{(s)} | X=x_a) P(Y|Y^{(s)},X=x_a)\\
&=P(Y|Y^{(s)}=y^{(s)}_a,X=x_a),
\end{align*}
where the last equality is due to the neural network's deterministic nature, i.e., 
\begin{align*}
P(Y^{(s)}=y^{(s)}_a|X=x_a) = 1
\end{align*}
where $y^{(s)}_a$ is the latent self-supervised label prediction. Thus the probability is 0 otherwise.

Intuitively, this demonstrates that the attack is equivalent to using the classifier
\begin{align*}
P(Y | X=x_a, Y^{(s)} = y^{(s)}_a)
\end{align*}
to predict the label.

Next we show that our algorithm is equivalent to using the following classifier
\begin{align*}
P(Y | X=x_a, Y^{(s)} = y^{(s)}).
\end{align*}

Our algorithm finds a new input image $X^{(n)}$ that
\begin{align*}
     x^{(n)}_{\mathrm{max}} &= \argmax_{\x^{(n)}} P(\X^{(n)}=\x^{(n)}|\X=\x)P(\Y^{(s)}=\y^{(s)} |\X^{(n)}=\x^{(n)}) \\
     &= \argmax_{\x^{(n)}} P(\X^{(x)}=\x^{(n)}| \X=\x, \Y^{(s)}=\y^{(s)})
\end{align*}

Note that 
\begin{align*}
&P(\Y | \X=\x_a, \Y^{(s)} = \y^{(s)})\\
&= \sum_{\x^{(n)}} P(\Y|\x^{(n)})P(\x^{(n)}|\X=\x_a,  \Y^{(s)} = \y^{(s)}) \\
&\approx P(\Y|\X^{(n)} = \x^{(n)}_{\mathrm{max}})P(\X^{(n)} =  \x^{(n)}_{\mathrm{max}} |\X=\x_a,  \Y^{(s)} = \y^{(s)})
\end{align*}

The last formulation is our algorithm's inference procedure, where we first estimate $x^{(n)}_{\mathrm{max}}$ with adversarial image $x_a$ and self-supervised label $y^{(s)}$. We then predict the label $Y$ using our new image $x^{(n)}_{\mathrm{max}}$. We now have proved that our algorithm is equivalent to using 
\begin{align*}
P(\Y | \X=\x_a, \Y^{(s)} = \y^{(s)}).
\end{align*}

Here, we use the maximum a posteriori (MAP) estimate $X^{(n)}_{\mathrm{max}}$ to approximate the marginalization over $ \x^{(n)}$ because: first, sampling a large number of $X^{(n)}$ is computational expensive, second, our our shows the sampling is ineffective, lastly, our MAP estimation also produce a denoised image that can also be useful for other downstream tasks.

\subsection{Detailed Proof of Theorem 1}

\begin{theorem}
Assume the base classifier operates better than chance and instances in the dataset are uniformly distributed over $n$ categories. Let the prediction accuracy bounds be $P(\Y|\Y^{(s)}_a, \X_a) \in [b_1, c_1]$ and $P(\Y|\Y^{(s)}, \X_a) \in [b_2, c_2]$.
If the conditional mutual information  $I(\Y;\Y^{(s)}| \X_a) > 0$, we have $b_2 \geq b_1$ and $c_2 > c_1$, which means our approach strictly improves the bound for classification accuracy.
\end{theorem}

\begin{proof}
If $I(\Y;\Y^{(s)}| \X=\x_a) > 0$, we have:
\begin{align*}
    I(\Y;\Y^{(s)}, \X_a) > I(\Y;\Y^{(s)}_a, \X_a) = I(\Y; \X_a)
\end{align*}


We let the predicted label to be $\hat{\Y}$, we assume there are $n$ categories, and let the lower bound for prediction accuracy to be $Pr(\hat{\Y}=\Y) \geq 1 - \epsilon_p$. We define $H(\epsilon_p)=-\epsilon_p log \epsilon_p-(1-\epsilon_p)log(1-\epsilon_p)$.
Use the \emph{Fano's Inequality} \cite{fano}, we have 
\begin{equation}
    H(\Y|\X_a) \leq H(\epsilon_p) + \epsilon_p \cdot \log(n-1)
\end{equation}


\begin{equation}
     -\epsilon_p \cdot \log(n-1) \leq H(\epsilon_p) -H(\Y|\X_a) 
\end{equation}

We add $H(\Y)$ to both side

\begin{equation}
    H(\Y) - \epsilon_p \cdot \log(n-1) \leq H(\epsilon_p) + I(\Y;\X_a)
\end{equation}
because $I(\Y;\X_a) = H(\Y) - H(\Y|\X_a)$. 

Then we get 

\begin{equation}
\label{eq:fano}
    H(\epsilon_p) + \epsilon_p \log(n-1) \geq - I(\Y;\X_a) + H(\Y)
\end{equation}

Now we define a new function  $G(\epsilon_p) = H(\epsilon_p) + \epsilon_p log(n-1)$. Given that in classification task, the number of category $n \geq 2$. We know $ log(n-1) \geq 0$. Given that the entropy function $H(\epsilon_p)$ first increase and then decrease, the function $G(\epsilon_p)$ should also first increase, peak at some point, and then decrease.

We calculate the $\epsilon_p$ for the peak value via calculate the first order derivative $G'(\epsilon_p)=0$. By solving this, we have:

\begin{equation}
    \epsilon_p = 1 - \frac{1}{n}
\end{equation}
which shows that the function $G(\epsilon_p)$ is monotonically increasing when $\epsilon_p \in [0, 1 - \frac{1}{n}]$. 

Given that we know, the base classifier already achieves accuracy better than random guessing, thus the given classifier satisfies $\epsilon_p \in [0, 1 - \frac{1}{n}]$.
Now, the function $G(\epsilon_p) = H(\epsilon_p) + \epsilon_p log(n-1)$ is a monotonically increasing function in our studied region, which has the inverse function $G^{-1}$.


By rewritting the equation \ref{eq:fano} We then have 
\begin{equation}
    G(\epsilon_p) \geq - I(\Y;\X_a) + H(\Y)
\end{equation}

We apply the inverse function $G^{-1}$ to both side:
\begin{equation}
    \epsilon_p \geq G^{-1}(- I(\Y;\X_a) + H(\Y))
\end{equation}

\begin{equation}
    1- \epsilon_p \leq 1 - G^{-1}(- I(\Y;\X_a) + H(\Y))
\end{equation}

Note that $(1-\epsilon_p)$ is our defined accuracy, thus $b_1 = 1 - G^{-1}(- I(\Y;\X_a) + H(\Y))$

The above derivation also applies to $P(\Y|\Y^{(s)}, \X=\x_a) \in [a_2, b_2]$, thus $b_2 = 1 - G^{-1}(- I(\Y;\X_a, \Y^{(s)}) + H(\Y))$.

Since $ I(\Y;\Y^{(s)}, \X_a) > I(\Y; \X_a)$
Thus $b_2>b_1$, using our approach, the upper bound for robust accuracy is improved.

To prove the lower bound $a_2 \geq a_1$, we divide the joint set of $\Y^{(s)} \cup \X_a$ into set $\X_a$ and $(\Y^{(s)} \cup \X_a) - \X_a$, given the additional information from $(\Y^{(s)} \cup \X_a) - \X_a$, the accuracy will not get worse, thus the new lower bound $a_2$ should not be smaller than $a_1$.


\end{proof}

\subsection{Defense Aware Attack}

We derive the defense aware attack in our main paper in details. We make the latent label $y^{(s)}$ explicit in our notation.

The straight forward adaptive attack is to optimize the attack in an adversary way against the defense.

\begin{align}
    \boldsymbol{r} &= \argmin_{\boldsymbol{r}} \mathcal{L}_s(\x+\boldsymbol{r}, \y^{(s)}) \\
    \boldsymbol{\delta} &= \argmax_{\boldsymbol{\delta}} \mathcal{L}_c(\x+\boldsymbol{r}+\boldsymbol{\delta}, \y)
\end{align}

From the attacker perspective, the above optimization is not ideal, as it involves iterative optimization of two directions, thus the gradient estimated maybe not stable enough, even having gradient obfuscation \cite{obfuscated}. Following the standard constrained optimization attack practice from \cite{intriguing, CW}, the attacker reformulates the above equation as a constrained optimization problem:

\begin{align}
\label{eq:opt}
    \mathrm{maximize}\quad &\mathcal{L}_c(\x_{a}+\boldsymbol{r}, \y), \\
    \quad \text{s.t.} \quad &\mathcal{L}_s(\x_{a} + r, \y^{(s)}) \leq \epsilon' 
\end{align}
where $\epsilon'$ is the same value as the converged loss $\mathcal{L}_s$ for natural images. Intuitively, the attacker should  maximize the adversarial gain while respecting the self-supervised loss if they want to render our defense ineffective.
    

The above equation is equivalent to:
\begin{align}
\label{eq:opt}
    \mathrm{minimize}\quad &-\mathcal{L}_c(\x_{a}+\boldsymbol{r}, \y), \\
    \quad \text{s.t.} \quad &\mathcal{L}_s(\x_{a} + r, \y^{(s)}) -\epsilon' \leq 0
\end{align}

We can use the Lagrangian Penalty Method to derive the following:
\begin{equation}
     \mathcal{L}_l(\x_{a}, \lambda_s)=-\mathcal{L}_c(\x_{a}, \y) + \lambda_s (\mathcal{L}_s(\x_{a} + \boldsymbol{r}, \y^{(s)}) - \epsilon')
\end{equation}

Thus the optimal value for the attack $\x_{a}^*$ is:
\begin{equation}
    \x_{a}^* =  \min_{\x_{a}} \max_{\lambda_s\geq 0} \mathcal{L}_l(\x_{a}, \lambda_s)
\end{equation}
which is the primal.

Using the Weak Duality Theorem we have the following upper bound for the optimal solution of the above optimization problem~\cite{KKT2}:
\begin{equation}
    \x_{a}^* = \min_{\lambda_s\geq 0} \max_{\x_{a}} \mathcal{L}_l(\x_{a}, \lambda_s)
\end{equation}
Removing the negative sign, we have:
\begin{equation}
    \x_{a}^* = \max_{\lambda_s\geq 0} \max_{\x_{a}} (\mathcal{L}_c(\x_{a}, \y) - \lambda_s (\mathcal{L}_s(\x_{a} + \boldsymbol{r}, \y^{(s)}) - \epsilon')).
\end{equation}

which is equivalent to first maximizing the followings under different $\lambda_s$:
\begin{equation}
    \mathcal{L}_l(\x_{a}, \lambda_s)=\mathcal{L}_c(\x_{a}, \y) - \lambda_s \mathcal{L}_s(\x_{a}, \y^{(s)})
\end{equation}
And then select the $\lambda_s$ that yields the most damage with the lowest robust accuracy, and use the corresponding generated attack $\x_{a}^*$.
%

\section{Experimental Results}

\subsection{Defense Aware Adversarial Attack}
We show the numerical results for the defense aware attack in Table \ref{tab:adaptive_attack}. In addition to the 50 steps we used in our main paper, we also show results using 500 steps of adaptive attack. We apply 500 steps to the RO (robust overfitting method by Rice et al. \cite{rice2020overfitting}). The results clearly show that using more steps does not change the conclusion. 500 steps of attack achieve almost the same robust accuracy as the 50 steps baseline, which suggests that the attack is almost converged. In addition, our approach still efficiently improves the robust accuracy by over 2\%. Lastly,  the attacker needs a $\lambda_s=4$ in order to bypass our defense through the reverse attack, however, at a cost that the attack for classification task gets weaker by over 7\%, which itself helps our defense. Overall, even under more attack steps, our defense is still effective.





\begin{table*}[t]
\begin{center}
      \vspace{-3mm}
    \centering
    \small
    \begin{tabular}{l|c|c|c|c|c|c|c|c}
         \toprule
         & \multicolumn{8}{c}{Defense Aware Adversarial Attack}  \\
         $\lambda_S$ & 0 & 0.5 & 1  & 2 & 4 & 6 & 8 & 10 \\
         \midrule  
        RO \cite{rice2020overfitting} 50 steps Baseline & \textbf{52.40\%} & 53.81\% & 55.41\% & 57.81\% & 60.80\% & 62.62\% & 63.81\% & 64.58\%\\
        RO \cite{rice2020overfitting} 50 steps with Ours &\textbf{54.59\%} & 55.61\% & 56.75\% & 58.67\% & 60.81\% & 62.07\%  & 63.16\% & 63.68\% \\
        \midrule
        RO \cite{rice2020overfitting} 500 steps  Baseline & \textbf{52.23\%} & 53.47\% & 54.89\% & 57.07\% & 59.86\% &  61.89\% & 63.09\% & 63.90\% \\
        RO \cite{rice2020overfitting} 500 steps with Ours & \textbf{54.70\%} & 55.61\% & 56.68\% & 58.17\%  & 60.51\%  & 61.26\% & 62.40\% & 63.06\% \\
        \midrule
        Semi-SL \cite{unlabeled} 50 steps Baseline & \textbf{62.30\%} & 63.87\% & 65.38\% &  67.87\% & 70.60\% & 72.32\% & 73.52\% & 74.42\% \\
        Semi-SL \cite{unlabeled} 50 steps with Ours & \textbf{64.64\%} &  65.62\%  & 66.72\% & 68.58\% & 70.55\% & 71.43\% & 72.75\% & 73.19\% \\
         \bottomrule
    \end{tabular}
\end{center}
      \vspace{-4mm}
\caption{The robust accuracy for defense aware attack under different $\lambda_S$ setup. We increase the $\lambda_S$ of defense aware attack from none to 10, and show robustness accuracy on two robust model RO and Semi-SL on CIFAR-10 dataset. While increasing the value of $\lambda_S$ decrease the gain of Ours compared with Baseline, it also increase the robust accuracy of baseline methods. To achieve the best attack efficiency, the attack should use $\lambda_S=0$, the standard attack without considering fooling the self-supervised branch.} \label{tab:adaptive_attack}
\end{table*}

\subsection{Defending a Stand-alone Non-robust Network}

In our main paper, we apply our approach to a list of existing state-of-the-art models, now we show results on applying our approach to an undefended neural network. 

We train a PreRes-18 model on pure clean images without any adversarial training, which yields 0\% robust accuracy under $L_{\infty}=8/255$ adversarial attack. We then use our defense to reverse the attack via natural supervision. We achieve an improvement in robust accuracy of 34.4\%.

\subsection{Feature Input for Self-supervised Models} We investigate which layer's feature should be the input for the self-supervision model. We conduct an ablation study that read out the latent features from low to high layers. Results in Figure \ref{fig:whichlayer} show that read our feature from the top layer for self-supervision achieved the most robustness gain.

\begin{figure*}[t]
\centering
\includegraphics[width=0.7\textwidth]{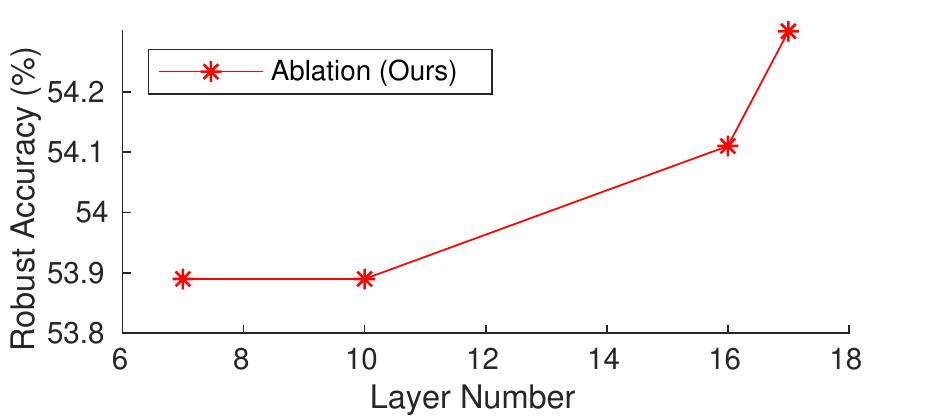}
  \caption{On CIFAR-10 dataset, we experiment on reading out from differnt layers and plot the robustness gain here.}
  \label{fig:whichlayer}
\end{figure*} 




\subsection{Implementation Details}

 We run our experiments with 8 RTX 2080 Ti GPU. For CIFAR-10, CIFAR-100, and SVHN, the input dimension is all $32 \times 32 \times 3$. We maximize the GPU space usage to speed up our inference. For the PreRes-18 model, we use a batch size of 1024 during inference. For the Wide Residual Network model, we use a batch size of 512. For the ImageNet dataset, the input size is $256 \times 256 \times 3$, we use the ResNet-50 model, and due to the larger input dimension and model capacity, we use a batch size of 87 to maximize the GPU usage. For all the contrastive learning, we sample 4 positive views for each given image instance.

\section{Visualization}

\subsection{Attack Vector Visualization:}
We visualize more examples of the adversarial attack vector and the inverse attack vector in Figure \ref{fig:viz_vec}. Our reverse attack vector is highly structured, reversing the mispredicted examples back to the right one. 

\subsection{Feature Visualization:}
We show more visualizations of the feature trajectories of our approach in Figure \ref{fig:pca_asa}, Figure \ref{fig:pca_ausau}, Figure \ref{fig:pca_cdc}, Figure \ref{fig:pca_dsd}, and Figure \ref{fig:pca_hrh}. We project the features onto a plane with PCA under the same setup as Figure 8 in the main paper. We can see that our approach pushes the misclassified examples (red) back to the original features (green), improving adversarial robustness.

\begin{figure*}[t]
\centering
\includegraphics[width=0.9\textwidth]{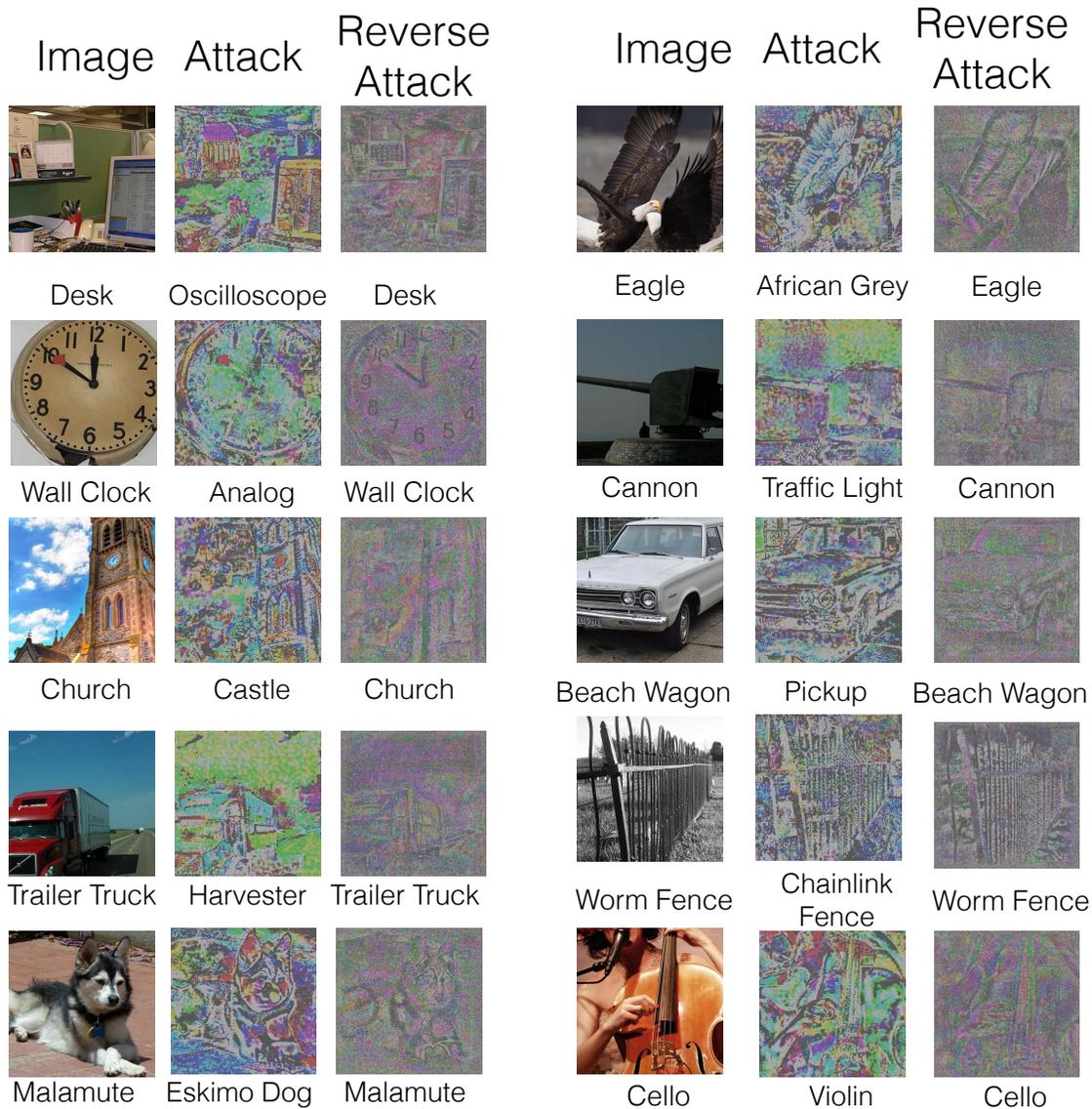}
  \caption{We show ImageNet's clean examples, attack vector, and our reverse attack vector. The adversarial attack is bounded by $\L_{\infty}=4/255$. By adding our reverse attack to the attack image, we can correct the misprediction on ImageNet classifier. As we can see, the reverse attack vector is also highly structured, which explains the reason that our approach is more efficient than adding random noise. The attack and reverse attack vectors have been multiplied by ten for visualization purposes only.}
  \label{fig:viz_vec}
\end{figure*}

\begin{figure*}[t]
\centering
\includegraphics[width=0.9\textwidth]{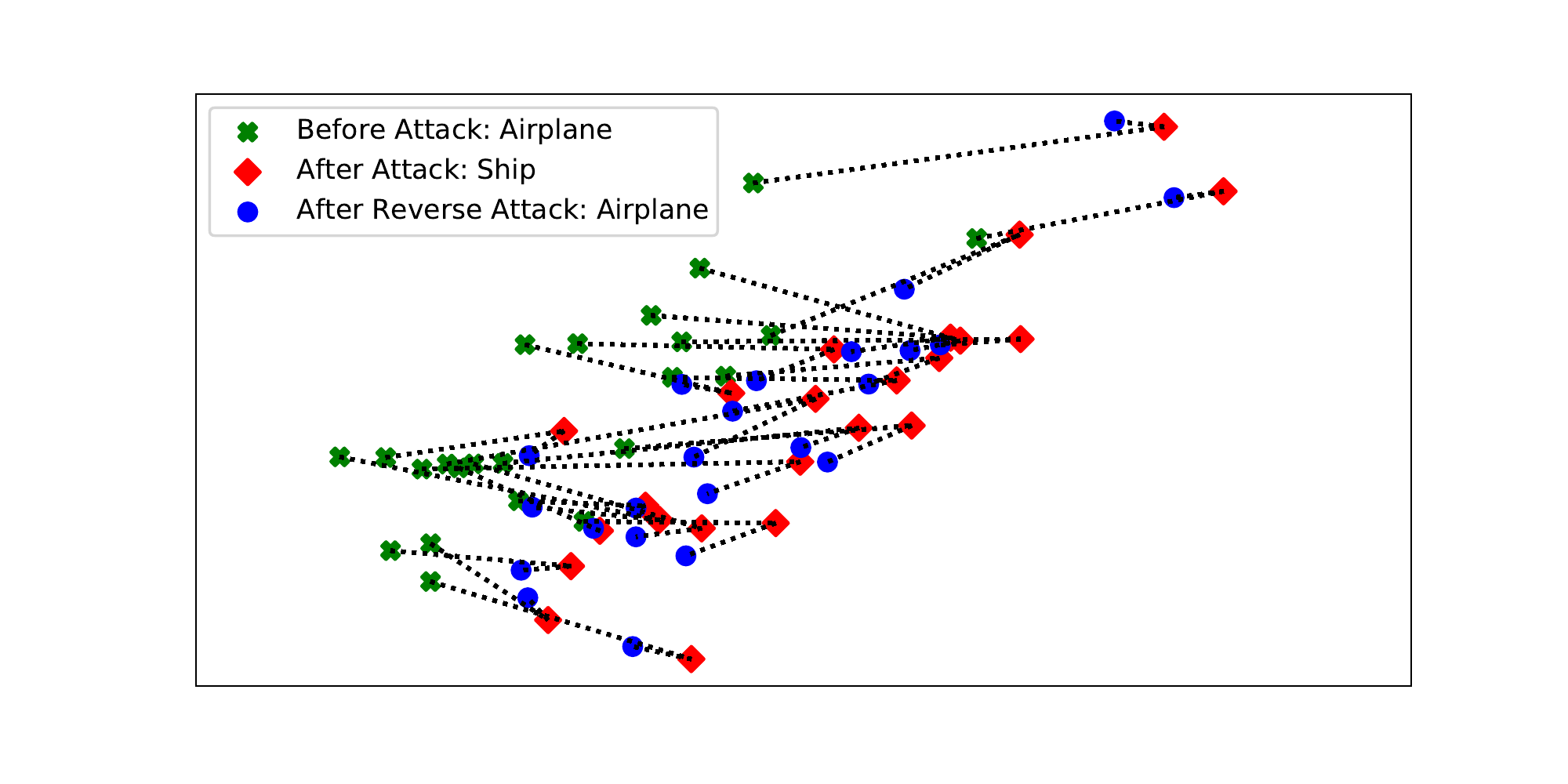}
  \caption{Feature Trajectories under attack and our reverse attack. We plot the figure in the same way as Figure 8 in the main paper with PCA.}
  \label{fig:pca_asa}
\end{figure*} 

\begin{figure*}[t]
\centering
\includegraphics[width=\textwidth]{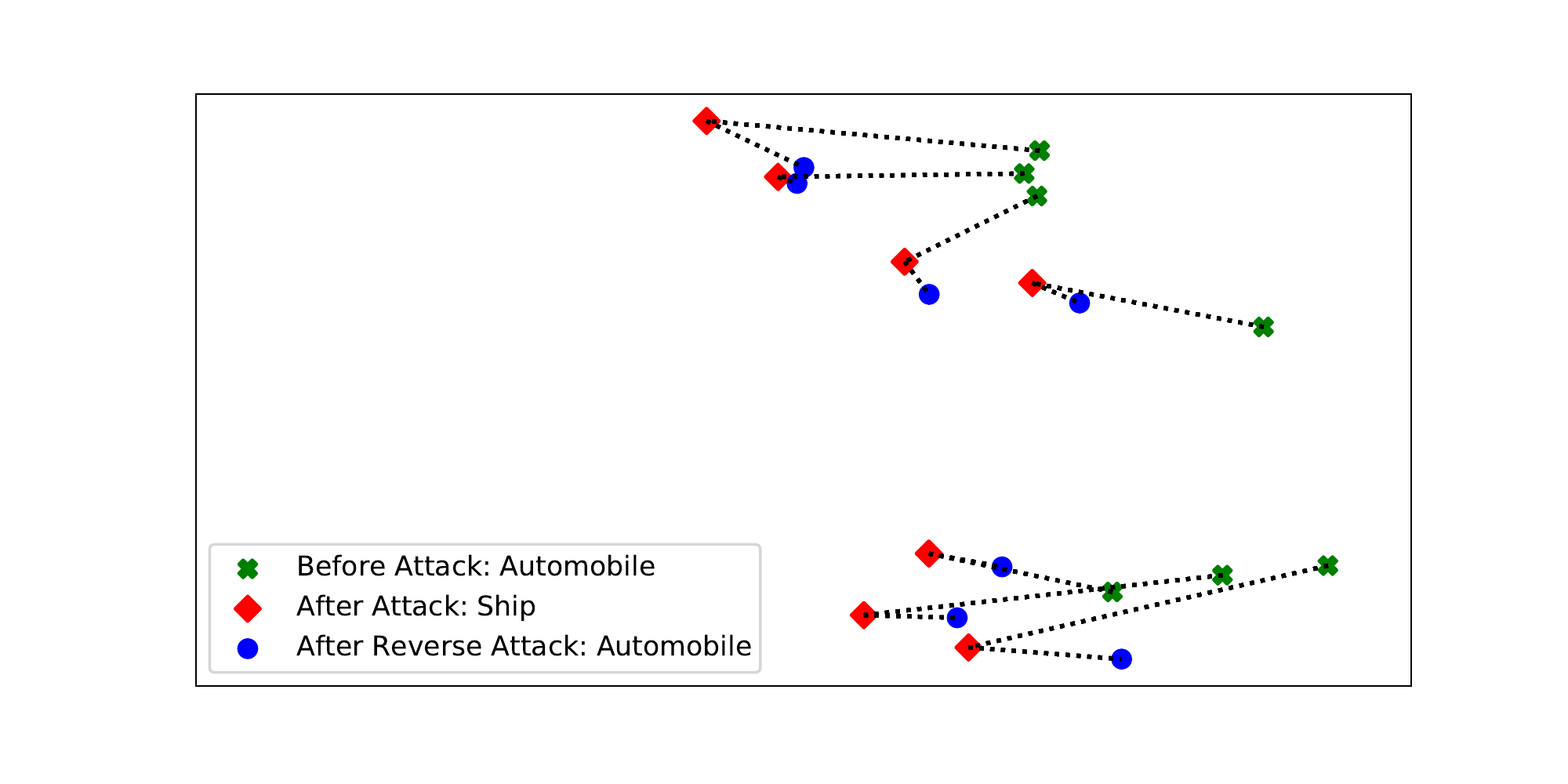}
  \caption{Feature Trajectories under attack and our reverse attack. We plot the figure in the same way as Figure 8 in the main paper with PCA.}
  \label{fig:pca_ausau}
\end{figure*} 

\begin{figure*}[t]
\centering
\includegraphics[width=\textwidth]{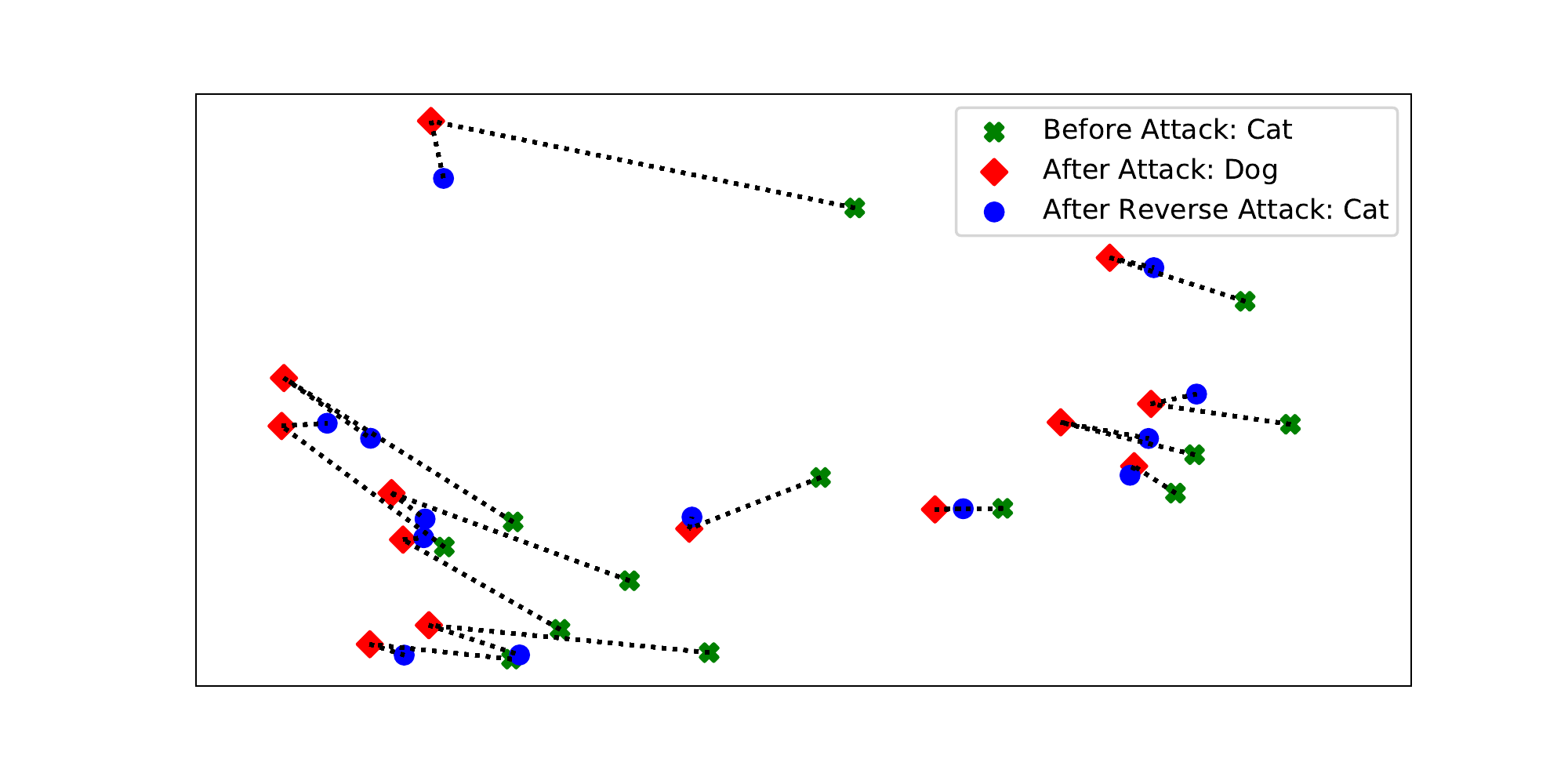}
  \caption{Feature Trajectories under attack and our reverse attack. We plot the figure in the same way as Figure 8 in the main paper with PCA.}
  \label{fig:pca_cdc}
\end{figure*} 

\begin{figure*}[t]
\centering
\includegraphics[width=\textwidth]{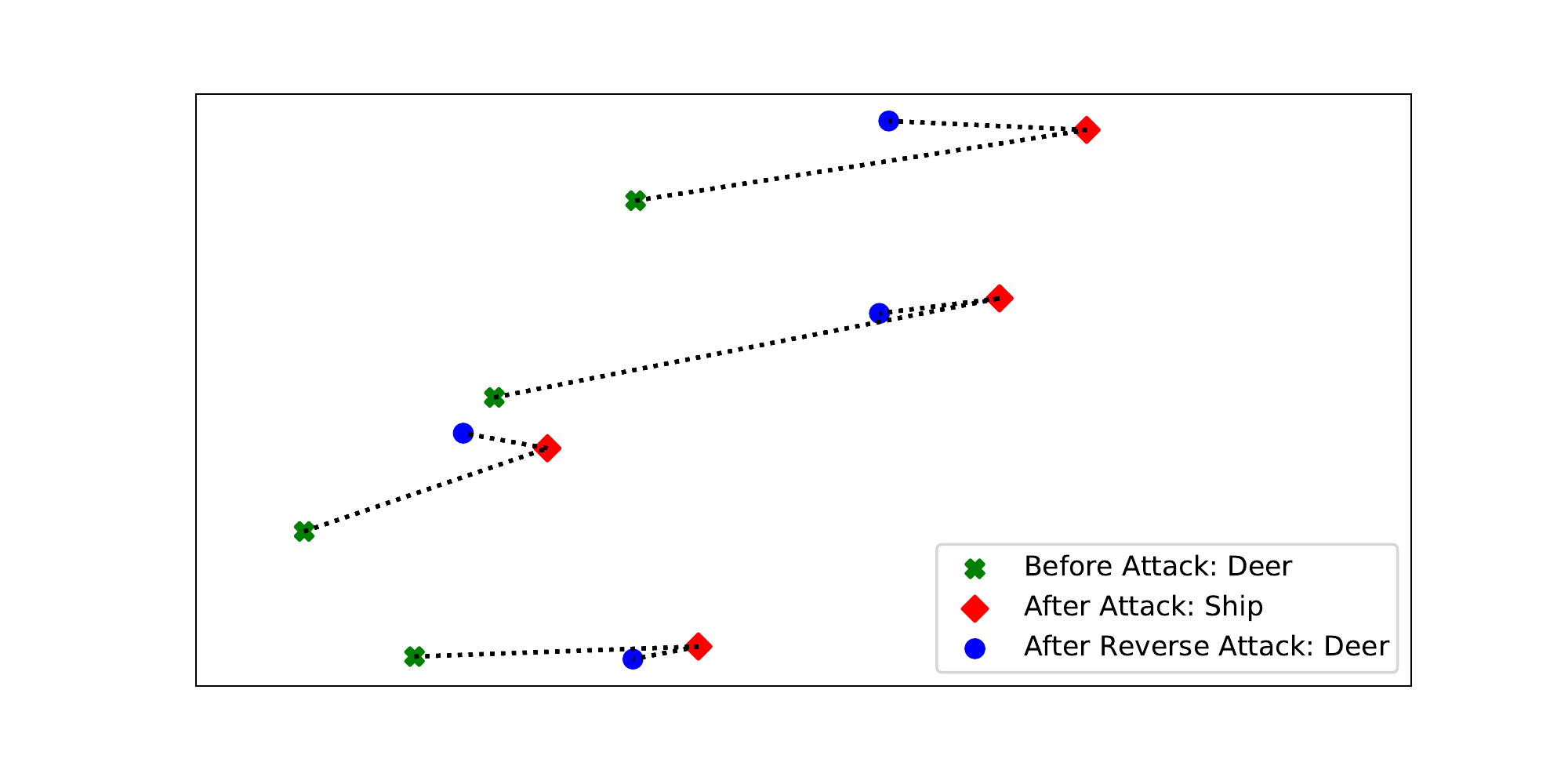}
  \caption{Feature Trajectories under attack and our reverse attack. We plot the figure in the same way as Figure 8 in the main paper with PCA.}
  \label{fig:pca_dsd}
\end{figure*}

\begin{figure*}[t]
\centering
\includegraphics[width=\textwidth]{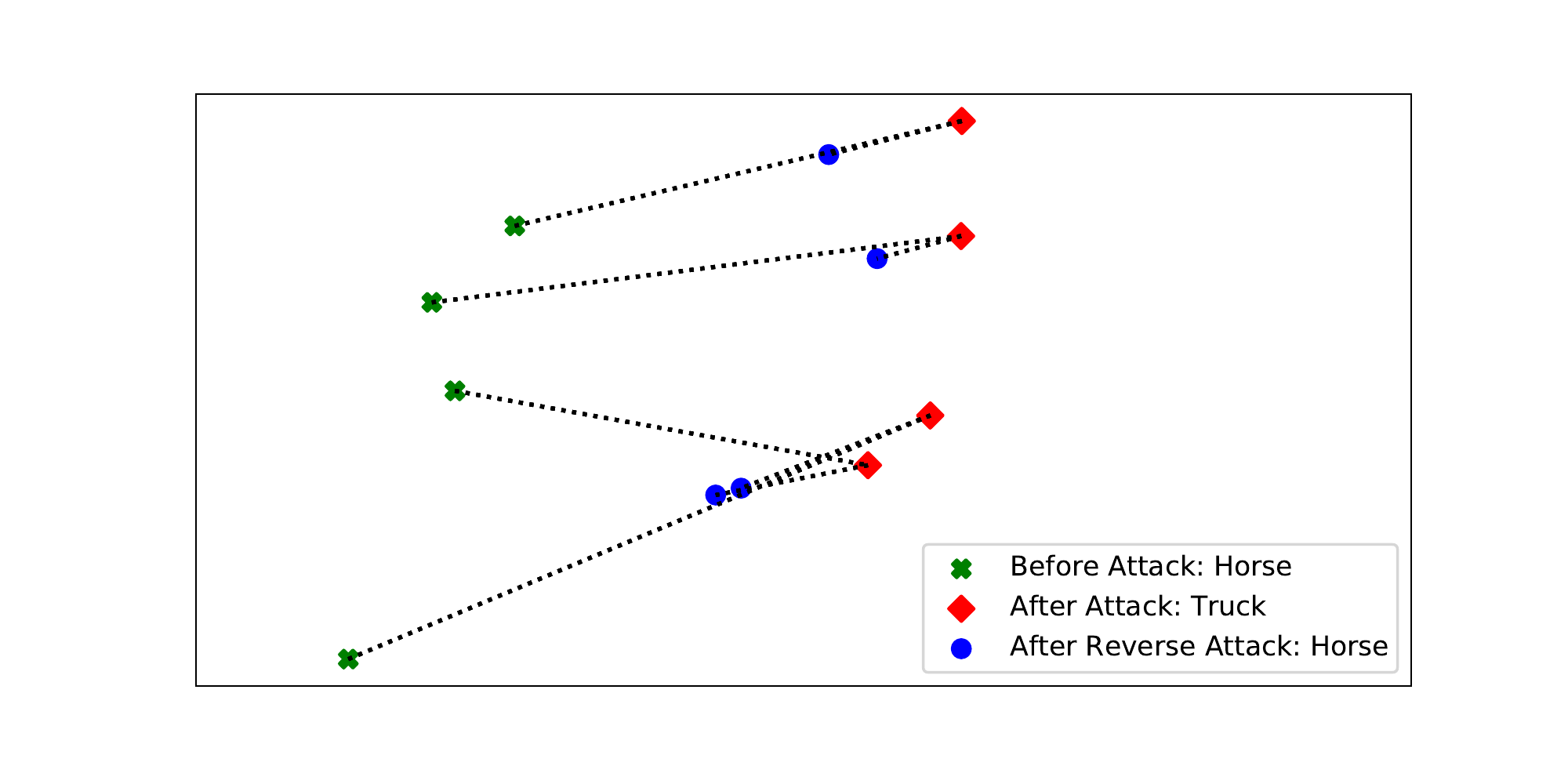}
  \caption{Feature Trajectories under attack and our reverse attack. We plot the figure in the same way as Figure 8 in the main paper with PCA.}
  \label{fig:pca_hrh}
\end{figure*} 

\clearpage

\bibliography{reference}
\bibliographystyle{plain}



\end{subappendices}


\end{document}